\documentclass{article}

\usepackage{microtype}
\usepackage{graphicx}
\usepackage{subfigure}
\usepackage{booktabs} %

\usepackage{hyperref}

\usepackage[accepted]{icml2024}

\usepackage{amsmath}
\usepackage{amssymb}
\usepackage{mathtools}
\usepackage{amsthm}
\usepackage{cancel}
\usepackage{dsfont}

\usepackage[capitalize,noabbrev]{cleveref}

\theoremstyle{plain}

\theoremstyle{definition}
\newtheorem{definition}{Definition}

\usepackage{mdframed}
\newmdtheoremenv{definition_frame}{Definition}

\usepackage[textsize=tiny]{todonotes}

\definecolor{myblue}{RGB}{0,115,230}
\definecolor{myred}{RGB}{181,25,99}
\definecolor{mygreen}{RGB}{91,163,0}

\newcommand{\tightdots}{.\kern-0.1em.\kern-0.1em.}
\newcommand{\tightcdots}{%
    \raisebox{0.5ex}{.}\kern-0.1em%
    \raisebox{0.5ex}{.}\kern-0.1em%
    \raisebox{0.5ex}{.}%
}

\newcommand{\ff}[1]{}
\newcommand{\zw}[1]{}
\newcommand{\zwT}[1]{}

\icmltitlerunning{Is In-Context Learning in Large Language Models Bayesian? A Martingale Perspective}

\usepackage{gensymb}

\usepackage{todonotes}
\usepackage{tikz-cd}
\usepackage{bbm}
\usepackage[shortlabels,inline]{enumitem}

\usepackage{multicol}

\usepackage{wrapfig}

\usepackage{lscape}

\usepackage{pgfplots}
\usepgfplotslibrary{groupplots}

\newcommand{\R}{\mathbb{R}}

\newcommand{\Eb}{\mathbb{E}}

\theoremstyle{definition}

\newtheorem{fact}{Fact}
\theoremstyle{remark}
\newtheorem{example}{Example}

\newcommand*\colourcross[1]{%
  \expandafter\newcommand\csname #1cross\endcsname{\textcolor{#1}{\ding{55}}}%
}
\colourcross{black}

\newif\ifdraft
\draftfalse %

\definecolor{MidnightBlue}{rgb}{0.1, 0.1, 0.44}

\usepackage{hyperref} 
\hypersetup{
  hyperfootnotes=true,
  hidelinks,
    pdftitle={},
    pdfauthor={},
    colorlinks=true,
    linkcolor=MidnightBlue,
    citecolor=MidnightBlue,
    filecolor=MidnightBlue,
    urlcolor=MidnightBlue,
    }
\usepackage{soul} %
\usepackage{bbding}
\usepackage{pifont}
\usepackage{caption}

\usepackage{graphicx}
\graphicspath{ {./Figures/} }

\usepackage{url}            %
\PassOptionsToPackage{hyphens}{url}

\usepackage{booktabs}       %
\usepackage{amsfonts}       %
\usepackage{amsmath}
\usepackage{amssymb}
\usepackage{nicefrac}       %
\usepackage{microtype}      %
\usepackage{wrapfig}

\usepackage{bm}
\usepackage{amsthm}

\usepackage{tikz}
\usetikzlibrary{bayesnet}
\usepackage{graphicx}
\usepackage{caption}

\usepackage{algorithm}

\usepackage{array,multirow,graphicx}
\usepackage{float}

\usepackage{wrapfig}

\usepackage{mathtools}

\newcommand{\mb}[1]{\mathbb{#1}}

\definecolor{darkgreen}{RGB}{1,50,32}

\newif\ifdraft
\draftfalse   %

\usepackage{enumitem}

\usepackage{thmtools}
\usepackage{thm-restate}

\begin{document}

\setcounter{proposition}{0} %
\setcounter{definition}{0} %
\setcounter{corollary}{0} %

\twocolumn[   

\icmltitle{Is In-Context Learning in Large Language Models Bayesian? \\ A Martingale Perspective}

\icmlsetsymbol{equal}{*}

\begin{icmlauthorlist}
\icmlauthor{Fabian Falck}{equal,yyy}
\icmlauthor{Ziyu Wang}{equal,yyy}
\icmlauthor{Chris Holmes}{yyy}
\end{icmlauthorlist}

\icmlaffiliation{yyy}{Department of Statistics, University of Oxford, Oxford, UK}

\icmlcorrespondingauthor{Fabian Falck}{fabian.falck@stats.ox.ac.uk}
\icmlcorrespondingauthor{Ziyu Wang}{ziyu.wang@stats.ox.ac.uk}
\icmlcorrespondingauthor{Chris Holmes}{cholmes@stats.ox.ac.uk}

\icmlkeywords{Machine Learning, ICML}

\vskip 0.3in
]

\printAffiliationsAndNotice{\icmlEqualContribution} %

\begin{abstract}
In-context learning (ICL) has emerged as a particularly remarkable characteristic of Large Language Models (LLM): given a pretrained LLM and an observed dataset, LLMs can make predictions for new data points from the same distribution without fine-tuning. Numerous works have postulated  ICL as approximately Bayesian inference, rendering this a natural hypothesis. In this work, we analyse this hypothesis from a new angle through the \textit{martingale property}, a fundamental requirement of a Bayesian learning system for exchangeable data. We show that the martingale property is a necessary condition for unambiguous predictions in such scenarios, and enables a principled, decomposed notion of uncertainty vital in trustworthy, safety-critical systems. We derive actionable checks with corresponding theory and test statistics which must hold if the martingale property is satisfied. We also examine if uncertainty in LLMs decreases as expected in Bayesian learning when more data is observed. In three experiments, we provide evidence for violations of the martingale property, and deviations from a Bayesian scaling behaviour of uncertainty, falsifying the hypothesis that ICL is Bayesian. 
\end{abstract}

\section{Introduction}
\label{sec:Introduction}

Large Language Models (LLMs) are autoregressive generative models trained on vast amounts of data, exhibiting extraordinary performance across a wide array of tasks \cite{zhao2023survey}.  %
A particularly remarkable characteristic of LLMs is so-called \textit{in-context learning} (ICL)  \cite{brown2020language,dong2022survey}: 
Given a pretrained language model $p_M$ and an observed dataset $D := \{(x_1,y_1), \ldots, (x_{n},y_{n})\} = z_{1:n}$ of samples, LLMs capture the distribution of the underlying random variables $X$ and $Y$ in this in-context dataset. %
This allows them produce a new sample $(x_{n+1}, y_{n+1})$ using the \textit{predictive distribution} $p_M(X_{n+1}, Y_{n+1} \vert Z_{1:n}=z_{1:n})$, or if $x_{n+1}$ is observed infer the predictive distribution $p_M(Y_{n+1} \vert X_{n+1}=x_{n+1},Z_{1:n}=z_{1:n})$, without retraining or fine-tuning $p_M$.

Few-shot learning via ICL \cite{brown2020language} has produced numerous breakthroughs in LLM research \cite{dong2022survey}, 
such as in supervised learning \cite{min2021metaicl} or  
chain-of-thought prompting \cite{wei2022chain}. 
In spite of the remarkable empirical success of ICL, we lack a unified understanding of the algorithm and the properties of conditioning LLMs on in-context data. %
In this work, we are interested in characterising the type of learning that occurs in ICL.  %
Specifically, we aim to answer the question: \textbf{is in-context learning for LLMs on exchangeable data (approximately) Bayesian?}  %

In contrast to prior work, our analysis focuses on one fundamental property of Bayesian learning systems for exchangeable data: the \textit{martingale property}. %
In a nutshell, the martingale property describes the invariance of a model's predictive distribution with respect to missing data from a population.
We will formally define and extensively explain the martingale property in \S \ref{sec:Background}, but begin by intuitively describing two important and desirable \textit{consequences} of it with an example,  highlighting its relevance. %
These consequences are:
(i) the martingale property is a necessary condition for rendering \textit{predictions unambiguous} in an \textit{exchangeable} data setting, and  %
(ii) it establishes a \textit{principled notion} of the model's \textit{uncertainty}.

\begin{figure*}[t]
    \centering
    \includegraphics[width=.9\linewidth,clip,trim={0 0.4cm 0 0.25cm}]{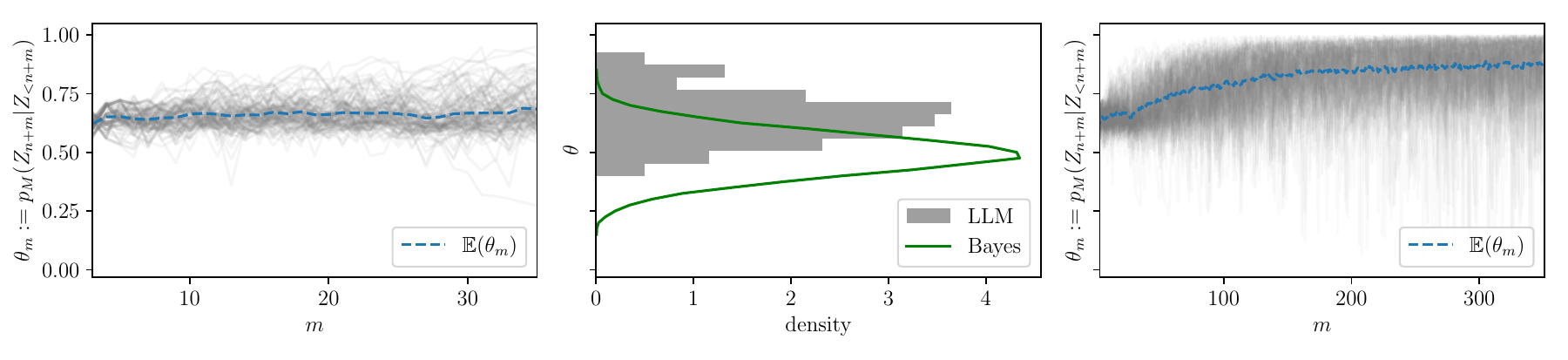}
\caption{
\textit{In-context learning in Large Language Models is not Bayesian.}
[Left] The \textit{martingale property}, a necessary condition of Bayesian learning systems, is satisfied for short sample paths.
[Centre] This allows us to %
approximate the \textit{martingale posterior} (see \S \ref{sec:The martingale property enables a principled notion of uncertainty}) which, however, indicates deviation from a reference Bayesian model.  %
[Right] For longer sample paths, we observe a drift which violates the martingale property, together rendering the ICL system non-Bayesian.  %
}
    \label{fig:figure1}
\end{figure*}

Consider a drug company exploring the efficacy of a new medication for headaches.
The company runs a two-arm Randomised Control Trial (RCT) with $100$ patients, 50 in each arm, comparing the new treatment with the current standard of care (in this case ibuprofen), and records the outcome $Y\in \{0,1\}$ whether patients are symptom-free four hours after treatment.
It is important to note that in this setting, the distribution of outcomes is independent of the order in which the patients are observed, a property known as \textit{exchangeability} (see \S \ref{sec:Background} for a formal definition). %
Half-way into the trial, the company conducts an interim analysis. 
Define the interim observations $D=\{(x_1, y_1), \ldots, (x_{50},y_{50})\}$ where $y_k$ indicates outcome, and $x_k$ the treatment arm and other patient covariates.
Given these observations, the company wants to decide whether to stop the trial early. %
The company uses an LLM, which was trained on potentially useful background information from the internet (e.g. on clinical trials, or the efficacy of ibuprofen), to generate the missing patients via ICL conditioning on $x_{1:n+k-1}$ for the $(n+k)$-th patient, and determines if the RCT is successful combining the observed and synthetic data.  %
It repeats this imputation procedure $J$ times, and decides to keep going with the trial if the fraction of symptom-free patients in the treatment over the control arm is above a certain threshold on average over these $J$ hypothetical trials.
\textit{Should we trust the LLM's prediction using ICL under this procedure?}   %

In preview of our experimental results in \S \ref{sec:Experiments}, the answer is \textit{`No'}.  %
Our experiments present evidence that state-of-the-art LLMs violate the martingale property in certain settings (see Fig. \ref{fig:figure1}).
The martingale property is a necessary condition for exchangeability, and in turn a fundamental property of Bayesian learning. %
If the martingale property is violated by an LLM performing ICL it implies that the model's predictions are not exchangeable, and hence that ICL with this LLM is not following any reasonable notion of probabilistic conditioning.
This renders the LLM's predictive distribution incoherent:   %
the model can make different predictions depending on the order in which the patients are imputed.
This is problematic because by the design of an RCT, we know that there is no outcome dependence on the order of observations.  %
It is incoherent and ambiguous to receive a different marginal predictions if we for example impute patient \# 51 or patient \# 100 first.
Note that independent and identically distributed (i.i.d.) is a stricter condition implying exchangeability, and hence our work also applies to any i.i.d.~data setting.  %
This should caution the practitioner of the use of LLMs in exchangeable applications and data settings.  %

But there is a second reason why the martingale property is crucial: 
it enables a principled interpretation of the \textit{uncertainty} of LLMs, allowing us to decompose inference into epistemic and aleatoric uncertainty (see \S \ref{sec:Background} for a detailed introduction).
Revisiting the RCT example above, 
if we acquire data from the $50$ remaining patients, a costly decision, can this substantially decrease (epistemic) uncertainty? 
What is the effect of acquiring additional features for each patient, e.g. a genetic predisposition, on the (aleatoric) uncertainty? --  %
Without satisfying the martingale property, we have no understanding of the effect on reducing uncertainty in applications where additional data acquisition is feasible, for instance active learning or reinforcement learning.   %
We cannot study the question `why is the point prediction of my LLM imprecise' in a principled way,  %
and the uncertainty of an LLM's predictive distribution remains opaque. 
This finding has important implications for safety-critical, high-stakes applications of LLMs where trustworthy systems with a principled uncertainty estimate are vital.

This work states the hypothesis that ICL in LLMs given exchangeable data is Bayesian.
Numerous works have argued that ICL approximates some form of Bayesian inference \cite{xie2021explanation,hahn2023theory,akyurek2022learning,zhang2023and,jiang2023latent} which we will carefully review in App.~\ref{sec:Related work}, rendering this hypothesis natural.  %
Our work introduces a novel perspective  which contradicts their conclusion: %
we show that the martingale property, a fundamental property of Bayesian learning systems,
is violated for state-of-the-art LLMs such as 
Llama2, Mistral, GPT-3.5 and GPT-4.   %
We on purpose focus our analysis on three synthetic experiments where the ground-truth data generating process is simple and known, and which provide a useful test bed without the convolution of unknown latent effects as is typical in natural language. 
Our goal is to provide a scientific and precise framework which measures and quantifies the degree to which ICL of an LLM is Bayesian. 

More specifically, our \textit{contributions} are: %
(a) 
We motivate the martingale property as a fundamental property of Bayesian learning, crucial for unambiguous predictions of an LLM in exchangeable settings, and a principled interpretation of uncertainty in LLMs (\S \ref{sec:Background}).  %
(b) 
We  derive actionable diagnostics with corresponding theory and test statistics of the martingale property for ICL.  %
We also characterise the efficiency of ICL compared to standard Bayesian inference (\S \ref{sec:sec3}).  %
(c)
We provide novel evidence for violations of the martingale property through LLMs in certain settings, and a deviation of the sample efficiency of ICL relative to Bayesian systems, falsifying our hypothesis that ICL in LLMs is Bayesian and cautioning against the use of LLMs in exchangeable and safety-critical applications (\S \ref{sec:Experiments}).  %

\section{What Characterises a Bayesian Learning System? A Martingale Perspective}
\label{sec:Background}

In this section we rigorously formalise properties of an ICL system that follows Bayesian principles.  %
Theoretical details and technical proofs are presented in App.~\ref{app:proof}.

\subsection{The Martingale Property}
We begin by defining the \textit{martingale property}. %

\begin{definition}%
The predictive distributions for $\{Z_i\}$ satisfy the \textit{martingale property} if for all integers $n,k>0$ and realisations $\{z, z_{1:n}\}$ we have 
\begin{align} 
&p_M(Z_{n+1}\!\!=\!\!z \vert Z_{1:n}\!\!=\!\!z_{1:n}) \!=\! p_M(Z_{n+k}\!\!=\!\!z \vert Z_{1:n}\!\!=\!\!z_{1:n}).   \label{eq:martingale}%
\end{align}
\end{definition}
Eq.~\eqref{eq:martingale}
states that $\{Z_i \} \sim p_M$ are \textit{conditionally identically distributed}  (\citealp{Berti2004}). 
As we will explain in \S\ref{sec:The martingale property enables a principled notion of uncertainty}, this renders  %
distributions $\{p_M(Z_{n+1}=\cdot\vert Z_{1:n})\}$ to form a martingale, hence the name `martingale property'. %

It follows from Eq.~\eqref{eq:martingale} that predictive distributions of the form $p_M(Y_{n+k}|X_{n+k}, Z_{1:n})$ satisfy a similar identity:  %
\begin{align} 
    &\phantom{{}={}}p_M(Y_{n+1}=y\vert X_{n+1}=x, Z_{1:n}=z_{1:n}) \nonumber\\ 
    &= p_M(Y_{n+k}=y | X_{n+k}=x, Z_{1:n}=z_{1:n}) \nonumber \\ %
    &= \Eb_{Z_{n+1:n+k-1}\sim p_M(\cdot|Z_{1:n}=z_{1:n})}  \nonumber\\
    &\hspace{3em} p_M(Y_{n+k}=y \vert X_{n+k}=x, Z_{1:n+k-1}),  \label{eq:martingale-y}
\end{align}
for all integers $n,k>0$, realisations $\{z_{1:n}, y\}$, and (almost every) realisation $x$ measured by $p_M(X_{n+1}|Z_{1:n}=z_{1:n})$.
In Eq.~\eqref{eq:martingale-y} the martingale property renders a model's predictions invariant to imputations of missing samples from the population (on average).
Note that %
Eqs. \eqref{eq:martingale} and \eqref{eq:martingale-y} are equivalent in the unconditional case ($x_i=\varnothing$), which we consider in the majority of our experiments in \S \ref{sec:Experiments}. %

\begin{figure}[t]
    \centering
    \includegraphics[width=.7\linewidth,trim={0 0.1cm 0 0.1cm}]{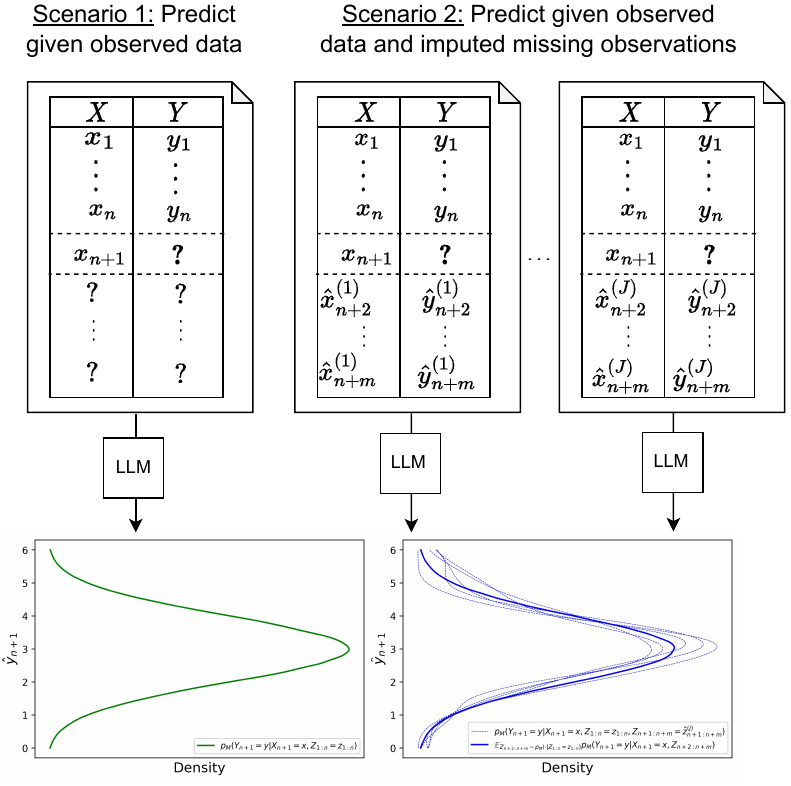}  
    \caption{
    The \textit{martingale property}, a fundamental requirement of a Bayesian learning system, requires \textit{invariance with respect to missing samples from a population.} 
    }
    \label{fig:martingale}
\end{figure}

\subsection{The Martingale Property is %
Necessary %
for Unambiguous Predictions under Exchangeable Data}
\label{sec:martingale}

To understand the %
intuition behind the seemingly technical notion of the martingale property, consider two scenarios for ICL, illustrated in Fig.~\ref{fig:martingale}.
In both scenarios, the LLM is given the observed data $(D, x_{n+1})$. %
In scenario 1, %
the LLM directly infers the \textit{predictive distribution} $p_M(Y_{n+1} \vert Z_{1:n}=z_{1:n}, X_{n+1} = x_{n+1})$. 
In scenario 2, before making a prediction, the LLM generates (imputes) $m-1$ missing samples $\hat{z}_{n+2:n+m}$ from the population autoregressively. %
Given the observed data and the imputed samples as a prompt, we then sample from %
the LLM's predictive distribution 
$p_M(Y_{n+1} \vert Z_{1:n}=z_{1:n}, X_{n+1} = x_{n+1}, Z_{n+2:n+m} = \hat{z}_{n+2:n+m})$. 
We repeat this imputation procedure $J$ times and average the obtained predictive distributions to receive a Monte Carlo estimate of the right-hand side of Eq.~\eqref{eq:martingale-y}. %
Scenario 2 is of practical interest when estimating aggregated statistics of a population as illustrated in our RCT example in \S \ref{sec:Introduction}. -- %
The martingale property then states that the predictive distribution from scenario 1, $p_M(Y_{n+1} \vert Z_{n}=z_n, X_{n+1} = x_{n+1})$, and the predictive distribution from scenario 2, $p_M(Y_{n+1} \vert Z_{n}=z_n, x_{n+1}, Z_{n+2:n+m}=\hat{z}_{n+2:n+m})$, when averaged over all possible imputations of $\hat{z}_{n+2:n+m}$ are equivalent.

Why is the martingale property natural for any probabilistic system, and LLMs in particular?   %
It is important to observe that all information about the distribution of $X$ and $Y$ presented to the model (in addition to its prior belief \cite{zellner1988optimal}) %
lies in the observed data $(D, x_{n+1})$. 
Imputing the samples $\hat{z}_{n+2:n+m}$ should hence not change the predictive distribution for $y_{n+1}$ when averaged over all possible imputations. 
This is precisely the core idea of the martingale property.
If the predictive distribution for $y_{n+1}$ changes on average, the %
model is `creating new knowledge' when there is none: %
it is `hallucinating'.   %
In preview of our experimental results in \S \ref{sec:Experiments}, we observe this violation of the martingale property in state-of-the-art LLM families. %
We call this phenomenon \textit{introspective hallucinations}: by querying itself, the model changes its predictions (on average), which as we shall see in \S\ref{sec:bayesian_link} violates how Bayesian systems learn.  %

There is another way in which predictions are rendered unambiguous: 
under \textit{exchangeability} for which the martingale property is a necessary condition (see App. \ref{app:proof}) 
the model is invariant to the order of the observed and missing data. %
This requirement is vital if we know that the order of the underlying distributions is irrelevant, for instance because---as in the RCT example in \S \ref{sec:Introduction}---we have designed the experiment such that we can exclude a dependency on the order.
Formally, this concept is known as exchangeability. 
A sequence of random variables $\{Z_i\}\sim p_M$  is \textit{exchangeable} if for all $\ell\in\mb{N}$ and $\ell$-permutations $\sigma$, %
\begin{align}
&\phantom{{}={}}p_M(Z_{1},\ldots,Z_{\ell})    %
= p_M(Z_{\sigma(1)},\ldots, Z_{\sigma(\ell)}). \label{eq:exc}  %
\end{align}

Exchangeability guarantees the invariance of predictions to the ordering of the observations $Z_{1:n}$, but also with respect to the order of future imputations $Z_{n+1,\ldots}\vert Z_{1:n}$. 
In the standard ICL setup, it is natural to assume that the sequence of example tuples in the ICL dataset, which is part of the prompt, is i.i.d.~and thus exchangeable, and many influential works make this assumption (often without stating it explicitly) \cite{xie2021explanation,wang2023large,jiang2023latent}.
To understand the importance of this assumption further, consider   %
the RCT example in \S \ref{sec:Introduction}, where %
$\{Z_{1:100}\}$ are (by experimental design) exchangeable. %
A model $p_M$ should hence satisfy   %
\begin{align*}
\small
&p_M(Y_{n+k} \vert X_{n+k}\!=\!x, Z_{1:n},%
X_{n+1:n+k-1}\!=\!\hat x_{n+1:n+k-1}) = \\   %
 & p_M(Y_{n+k} \vert X_{n+k}\!=\!x, Z_{1:n},%
 X_{n+1:n+k-1}\!=\!\hat x_{\sigma(n+1:n+k-1)}),  %
\end{align*}
meaning that the %
prediction for $Y_{n+k}\vert D, X_{n+k}$ is independent of the order of the imputed inputs $\hat x_{n+1:n+k-1}$.  
If a model $p_M$ violates the above equality, 
there may be ambiguities in the prediction of the next sample $(Y_{n+k}, X_{n+k})$ as it may depend on and vary with the ordering. %
Such ambiguities would 
substantially undermine the credibility of %
predictions, as well as the downstream decision-making based on such procedures. %
The martingale property is connected to the above notions of invariance as a necessary condition for exchangeability. %
Furthermore, it can even ensure exchangeability of imputed samples as the observed sample size $n$ becomes large, %
because Eq.~\eqref{eq:martingale} implies asymptotic exchangeability of $Z_{n+1,\ldots}\vert Z_{1:n}$ %
\citep[Thm.~2.5]{Berti2004}. %

\subsection{The Martingale Property Enables a Principled Notion of Uncertainty}
\label{sec:The martingale property enables a principled notion of uncertainty}

The second desirable and important consequence of the martingale property is that it establishes a principled notion of uncertainty in the model's predicitive distribution.
More specifically, it allows us to decompose this uncertainty, enabling us to study and interpret the uncertainty of a model.

To simplify the exposition, suppose the variables $Z_i$ are discrete and have $A<\infty$ realisations (both standard in LLMs) \footnote{We refer to App.~\ref{app:mp} for a review of the more general case.}, %
 so that 
any distribution  $p_\theta(Z=\cdot)$ can be identified by 
a vector $\theta\in \R^A$. %
Let $\theta_n$ denote the random vector that indexes $p_M(Z_{n+1}\mid Z_{1:n})$. %
Then, the martingale property is equivalent to %
stating that $\{\theta_n\}$ form a martingale  
w.r.t.~the filtration defined by $\{Z_n\}$. %
Under boundedness conditions always satisfied in the above %
case, Doob's theorem \cite{doob1949application} states that $\theta_n$ converges almost surely to a random vector $\theta_\infty$, and we have 
$
\theta_n = \Eb_{\theta_\infty\mid Z_{1:n}} \theta_\infty,  %
$
or equivalently, 
\begin{equation}
\label{eq:uncertainty}
p_M(Z_{n+1}\!=\!\cdot \!\mid\! Z_{1:n}) \!=\!\! \int\! {\color{myblue}p(\theta_\infty\vert Z_{1:n})} {\color{myred}p_{\theta_\infty}(Z\!=\!\cdot)}  d\theta_\infty.
\end{equation}
Note the similarity of Eq.~\eqref{eq:uncertainty} with %
Bayesian inference: 
the Bayesian posterior %
predictive distribution has the form 
\begin{equation}\label{eq:bayes-uncertainty}
p_M(Z_{n+1}=\cdot \mid Z_{1:n}) = \int {\color{myblue}p(\theta\vert Z_{1:n})} {\color{myred}p(Z=\cdot\vert \theta)}  d\theta.
\end{equation}
The random vector $\theta_\infty$ plays the same role as the %
parameter $\theta$ in a Bayesian model, as    %
both %
determine a predictive distribution ($p_{\theta_\infty}(Z)$ or $p(Z\vert\theta)$). %
They are thus interchangeable for prediction purposes.
Moreover, 
if %
$p_M$ is defined through Bayesian inference over %
$\theta$, %
$p_{\theta_\infty}$ will define the same distribution over $Z$ as  $p(\cdot\vert\theta)$
(see App.~\ref{app:mp}). %
Therefore we refer to the %
distribution $\theta_\infty\vert Z_{1:n}$ as the \textit{martingale posterior}.

Eq.~\eqref{eq:uncertainty} shows that 
the variation or \textit{uncertainty} in the predictive distribution 
$p_M(Z_{n+1}=\cdot \mid Z_{1:n})$ has two sources: 
\begin{enumerate}[topsep=0pt, itemsep=0pt, partopsep=0pt, parsep=0pt,leftmargin=*]
    \item \textbf{\color{myblue}epistemic uncertainty}, which is %
    about the latent %
    $\theta_\infty$ and %
    can be reduced if more data is available; and  %
    \item \textbf{\color{myred}aleatoric uncertainty}, which is irreducible given a fixed set of features even if infinite samples are observed and all aspects of the data generating process, namely the latent $\theta_\infty$, are known. 
\end{enumerate} 

The close connection between %
Eqs.~\eqref{eq:uncertainty} and \eqref{eq:bayes-uncertainty} %
shows that this decomposition of uncertainty is established by the same foundations as in Bayesian inference. 
This is particularly relevant for LLMs which lack clearly stated, interpretable and verifiable assumptions (such as a prespecified statistical model), %
rendering %
their predictive distribution %
a `black-box'. 

Importantly, we can %
construct the martingale posterior solely using path samples from $p_M$:
we can sample from $p(\theta_{n+k}|Z_{1:n})$ simply by sampling $Z_{n+1:n+k-1}|Z_{1:n}$ as $\lim_{k\to\infty}\theta_{n+k}=\theta_\infty$. 
Alternatively, we can also estimate parametric models on the path samples as proposed in \citet{fong2021martingale} (see App.~\ref{app:mp} for further details). %
This construction is an appealing tool for interpreting black-box models such as LLMs.   %

The interpretable decomposition of uncertainty further provides actionable guidance on how the combined uncertainty can be reduced: 
We can collect more samples to reduce epistemic uncertainty in scenarios where this is possible such as active learning, reinforcement learning or healthcare; 
particularly in regions of the input space where the uncertainty is high.
In \S \ref{sec:additional-checks} we propose diagnostics to check if epistemic uncertainty decreases w.r.t.~training sample size.
On the contrary, if the aleatoric uncertainty is high and ought to be reduced, we cannot do so without `changing the problem', for instance by collecting more features for each data point.
This principled notion of uncertainty in a model is crucial in safety-critical, high-stakes scenarios for building trustworthy systems.%

We present the following example for further intuition: 
\begin{example}\label{ex:bernoulli}
Suppose $Z_i\in\{0, 1\}$. Then $\theta_\infty = (\theta_{\infty,0},\theta_{\infty,1})\in\R^2$, and $p_{\theta_\infty} = \mathrm{Bern}(\theta_{\infty,1})$. 
Thus, in both Eq.~\eqref{eq:uncertainty} and Eq.~\eqref{eq:bayes-uncertainty} the epistemic uncertainty is represented by a distribution over the Bernoulli parameter, revealing their inherent connection. 
The epistemic uncertainty is especially important in scenarios where we use a black-box model %
$p_M$ to impute the missing samples $\{Z_{n+i}\}$ from a population 
---as in the RCT example in \S\ref{sec:Introduction}--- 
and want to quantify a model's lack of knowledge about the population. 
Note this distribution is not identifiable if we only have samples from a single-step predictive distribution $p_M(Z_{n+1}\vert Z_{1:n})$, but becomes identifiable given \textit{sample paths}. 
\end{example}

\subsection{On the Link between the Martingale Property and Bayesian Learning Systems}
\label{sec:bayesian_link}

So far, we asserted that the martingale property is fundamental to a Bayesian ICL system.  %
In this subsection, we want to further formalise this.
We have already discussed the close connection between the martingale property, exchangeability (\S \ref{sec:martingale}), and uncertainty (\S \ref{sec:The martingale property enables a principled notion of uncertainty}).  %
We will now show that for ICL on i.i.d.~data, exchangeability, for which the martingale property is a necessary condition, and Bayesian inference are closely connected, equivalent conditions.  %

ICL typically assumes i.i.d.~observations $Z_{1:n}$, which is our primary focus in this work (see \S \ref{sec:martingale}). %
Therefore, a correctly specified Bayesian model should produce %
marginal predictive distributions %
of the form 
\begin{align}
&\phantom{{}={}}p_M(Z_{1:n}\!\!=\!\!z_{1:n})  \label{eq:invar_impl}\\
&\!=\!
\!\!\int\! p_M(Z_1\!\!=\!\!z_1,\ldots, Z_n\!\!=\!\!z_n\vert\theta) \pi(\theta)d\theta \nonumber\\ 
&\!=\! 
\!\!\int\! \biggl(\prod_{i=1}^n p_M(Z\!\!=\!\!z_i\vert\theta)\!\biggr)\pi(\theta)d\theta, ~~~\forall n\in\mathbb{N}. \label{eq:exc0}
\end{align}
Here, $\theta$ denotes the parameter of a Bayesian model, $\pi$ denotes the prior measure %
and $p_M(Z=\cdot\mid \theta)$ denotes the likelihood.   %
From the factorisation over the data dimension $n$ in \eqref{eq:exc0}, we can see that it is invariant with respect to permutations of $z_{1:n}$, and thus the left-hand side of the equation in \eqref{eq:invar_impl} is invariant, too.
It then follows that $\{Z_i\}\ \sim p_M$ satisfies Eq.~\eqref{eq:exc}, and thus $\{Z_i\}$ are exchangeable.   %
The converse is also true by %
de Finetti's representation theorem \cite{de1929funzione}: %
Under mild regularity conditions 
any $p_M$ that defines exchangeable  $\{Z_i\}$ must have a representation %
in the form of Eq. \eqref{eq:exc0}. 
It then follows that 
the predictive distribution $p_M(Z_{n+1}|Z_{1:n})$ %
has the form of a Bayesian posterior predictive distribution, 
\begin{align*}
p_M(Z_{n+1}|Z_{1:n}) = \int  p_M(Z_{n+1}\vert \theta) \pi(\theta | Z_{1:n})d\theta,
\end{align*}
and can thus be viewed as implicit Bayesian inference for the %
latent variable $\theta$ \cite{bayes_blog}.  %
In conclusion, ICL on i.i.d.~data corresponds to a Bayesian model that assumes (conditionally) i.i.d.~observations \textit{if and only if} it defines an exchangeable sample sequence.  %
Since the martingale property is a necessary condition for exchangeability, an ICL system not satisfying the martingale property cannot be Bayesian.  %

\section{Probing Bayesian Learning Systems through Martingales}
\label{sec:sec3}

In this section we introduce practical diagnostics to probe if LLMs match the behaviour of Bayesian learning systems. %

\subsection{Are All Deviations from Bayes Bad? -- Expected and Acceptable Deviations from Bayesian Reasoning}  %
\label{sec:acceptable-deviations}

Numerous properties are implied if a learning system satisfies the martingale property, a distributional characteristic, and it is both infeasible and unnecessary %
as often practically irrelevant to check all of them in order to provide evidence for or against our hypothesis.  %
For example, the martingale property implies that all conditional moments should be equivalent, i.e. $\Eb(Z_{n'+1}^l | Z_{1:n}) = \Eb(Z_{n'+k}^l | Z_{1:n})$ for all integers $n,n',k,l>0$ and $n'>n$, yet higher-order moments are not vital in most applications and hence are acceptable deviations, if existent. 
Therefore, we will restrict our attention to two key implications of the martingale property which---if present---have important practical consequences.

Pretrained LLMs are general-purpose models and can at best approximate Bayesian learning via ICL. 
The martingale property is an invariance that is not hard-coded in their transformer-based architecture, and can only be approximately (rather than exactly) satisfied.
Let us assume that an LLM internally maintains a `hierarchy of states' \cite{wang2023large}, say a hierarchical Bayesian model, capturing different tasks (e.g. Bayesian ICL from i.i.d.~data, or acting in a dialogue system), and at each sampling step first updates its belief about this state. 
Say there is a probability $p$ that the LLM deviates from Bayesian ICL or simply fails to approximate.  %
Even if $p$ is small, the probability of a deviation $1\!-\!(1\!-\!p)^m$ becomes substantial when accumulated over a long sampling path of length $m$.    %
In early experiments, we observed frequent poor approximations for long sampling paths (see Fig. \ref{fig:bern-hf-long} in the Appendix).
This would trivially falsify the martingale property and our hypothesis. 

In our experiments in \S \ref{sec:Experiments}, we hence restrict the sampling paths to a %
short, finite length where we check the martingale property.
We also design our checks to be robust against such behaviour, for example by removing outliers before computing a test statistic. %
Furthermore, we are particularly interested in stark and unequivocal evidence of the model violating the martingale property beyond an expected error of any approximating model.   %
We will analyse and quantify violations of the martingale property with diagnostics, which we introduce in \S \ref{sec:CID-check}, in order to check our hypothesis experimentally.
In App.~\ref{app:acceptable-dev} we derive the order of `acceptable violations' %
for the test statistics we will introduce. %

\subsection{Diagnostics for the Martingale Property}\label{sec:CID-check}

As we showed in \S \ref{sec:bayesian_link}, the martingale property is fundamental to a Bayesian learning system. 
In this work, we probe the martingale property in LLMs via two properties \textit{implied by} it.
If these implied properties are strongly violated, so is the martingale property.
More specifically, %
we will derive implications involving %
conditional expectations of the form $\Eb(f(Z_{n+1:n+m})\vert Z_{1:n})$, %
which can be estimated 
by generating sample paths $\{z_{n+1:n+m}^{(j)}\sim p_M(Z_{n+1:n+m}\vert Z_{1:n}=z_{1:n})\}_{j=1}^J$ autoregressively with an LLM, %
and %
use these samples %
to form Monte Carlo estimates of the conditional expectations. 
We begin with an equivalent characterisation of the (conditional) martingale property. %

\begin{restatable}{proposition}{prop} \label{lemma:side_2}
A sequence %
$\{Z_{n+1:n+m}\} \sim p_M(\cdot\vert Z_{1:n})$ 
satisfies the martingale property %
if and only if the following holds: for all $n',k\in\mb{N}$ and integrable functions $g,h$:
\begin{align}\label{eq:CID-equiv-form}
\Eb((g(Z_{n'+k})-g(Z_{n'+1})) h(Z_{n+1:n'}) | Z_{1:n}) = 0.
\end{align}
\end{restatable}

We now state two implications of \cref{lemma:side_2}, our two diagnostics of the martingale property, which we will check experimentally in \S \ref{sec:Experiments}.  %
\begin{restatable}{corollary}{martingaleimplications}
\label{prop:martingale_implications}
Let $\{Z_i:i\in\mb{N}\}$ be a sequence of random variables satisfying the martingale property. 
Then for all integers $n,n',k>0$ and $n'>n$ it holds that: %
\begin{enumerate}[itemsep=0pt, parsep=0pt, topsep=0pt]
    \item[(i)] $\Eb(g(Z_{n+1})\vert Z_{1:n}) = \Eb(g(Z_{n+k})\vert Z_{1:n})$ for all integrable functions $g$, and
    \item[(ii)] $\Eb((Z_{n'+k+1} - Z_{n'+1}) Z_{n'}^\top\vert Z_{1:n}) = 0$. 
\end{enumerate}
\end{restatable}

Properties (i) and (ii) %
are derived from Proposition~\ref{lemma:side_2} by making different choices of the functions $(g, h)$. 
Property (i) follows by setting $h(Z_{n+1:n'})\equiv 1$ %
and examines the marginal predictive distributions %
$p_M(Z_{n+k}\vert Z_{1:n})$. 
We instantiate (i) using (at most) two choices of $g$: %
In preview of \S \ref{sec:Experiments}, we will perform our checks on unconditional experiments where $Z_i$---or equivalently $Y_i$ because of the unconditional setting---are Bernoulli or Gaussian distributed random variables. %
In the Bernoulli experiment it suffices to choose the identity function $g(z)=z$, 
as the mean $\Eb(Z_{n+k}\vert Z_{1:n})$ provides full information about the distribution $p_M(Z_{n+k}\vert Z_{1:n})$. 
In the Gaussian experiment, we will observe that choosing $g(z)=z$ and $g(z)=z^2$ is in most cases sufficient to reveal substantial violations from the martingale property. 

Property (ii) is equivalent to requiring Eq. \eqref{eq:CID-equiv-form} to hold for all linear functions $(g,h)$, which follows by linearity of the functions and the conditional expectation.
We will again see in our experiments that this choice is usually sufficient to reveal deviations from the martingale property. 
Let us further consider our choices for $h$ and $g$ with an example. %

\begin{example}\label{ex:corr-1-and-mean-param}
Suppose $p_M$ is a Bayesian learning system over a latent parameter $\theta$ (see Eq.~\eqref{eq:exc0}), and the respective likelihood $p(Z\vert\theta)$ satisfies $\Eb_{Z\sim p(Z\vert\theta)}Z=\theta$. 
Then by \cref{prop:martingale_implications}, for all $(k,n')$ we have 
\begin{itemize}[itemsep=1pt, parsep=1pt, topsep=1pt,leftmargin=*]
    \item 
$\Eb(Z_{n+k}\vert Z_{1:n}) = \Eb(\theta\vert Z_{1:n})$, and
\item $\Eb(Z_{n'+k+1}Z_{n'+1}^\top\vert Z_{1:n}) = \Eb(\theta\theta^\top\vert Z_{1:n})$ %
\citep[see e.g.][p.~454]{ghosal2017fundamentals}.
\end{itemize}

In this setting, condition (i) (with $g(z)=z$) and (ii) 
thus guarantee that the conditional mean and covariance equal the posterior mean and covariance, respectively, independent of the indices $(n',k)$. 
These two important aspects of the posterior are hence consistently expressed by the model.
The example is especially relevant %
as it covers Bernoulli ($p(Z \vert\theta)=\mathrm{Bern}(\theta)$) and Gaussian data, which will be our main focus in the experiments. 
\end{example}

In App. \ref{app:Additional experimental details and results} we present aggregated statistics $T_{1,g}$ and $T_{2,k}$ to compute and empirically measure properties (i) and (ii) from sample paths generated by an LLM.
In our experiments, we check if these statistics lie within bootstrapped confidence intervals obtained by a reference Bayesian predictive model, which is readily available in synthetic settings, through the same sampling procedure. 
We will refer to these comparisons as `checks' of the martingale property.
If  $T_{1,g}$ and $T_{2,k}$ lie outside the confidence interval, properties (i) and (ii) and hence the martingale property are violated.

\subsection{Diagnostics for Epistemic Uncertainty}\label{sec:additional-checks}

As discussed in \S \ref{sec:The martingale property enables a principled notion of uncertainty}, the martingale property allows us to identify epistemic uncertainty, which should decrease with more observed samples.
Here, we derive a third diagnostic for Bayesian ICL systems which probes this.  %
We begin by presenting a theoretical fact which provides important intuition on the role of epistemic uncertainty. %
\begin{restatable}{fact}{posteriorvarianceesterror}
\label{fact:posterior-variance-est-error}%
Let $\pi(\theta)$ and $p_M(Z\vert\theta)$ be the prior and likelihood of a Bayesian model,
$\bar\theta_n := \Eb_{\theta\sim \pi(\theta\mid z_{1:n})} \theta$ the posterior mean given data %
$z_{1:n}$, 
and $\|\cdot\|$ be any vector norm. Then, %
\begin{align} 
&\phantom{{}={}}\Eb_{\theta_0\sim \pi, z_{1:n}\sim \pi(z\mid\theta_0)} \Eb_{\theta\sim \pi(\theta\mid z_{1:n})} \|\theta - \bar\theta_n\|^2   
\nonumber\\&=\Eb_{\theta_0\sim \pi, z_{1:n}\sim \pi(z\mid\theta_0)} %
\|\theta_0 - \bar\theta_n\|^2.  \label{eq:fact}
\end{align} 
\end{restatable}
The left-hand side in Eq.~\eqref{eq:fact} is the trace of the posterior covariance (variance) and thus measures epistemic uncertainty. 
The right-hand side is the  %
estimation error for the true parameter. %
Thus, \Cref{fact:posterior-variance-est-error} states that \textit{epistemic uncertainty provides a quantification for the average-case estimation error}. %
Note that Eq.~\eqref{eq:fact} only applies to data from the prior predictive distribution, %
and thus not necessarily to the real observations. %
Nonetheless, %
a significant deviation of a model from the known scaling behaviour of the estimation error will indicate non-conformance with any reasonable Bayesian models. 
This is precisely our starting point to derive another diagnostic for Bayesian ICL systems.

\begin{figure*}[t] %
    \centering 

    \includegraphics[width=0.495\textwidth,trim={0.9cm 0.2cm 0 0.5cm}]{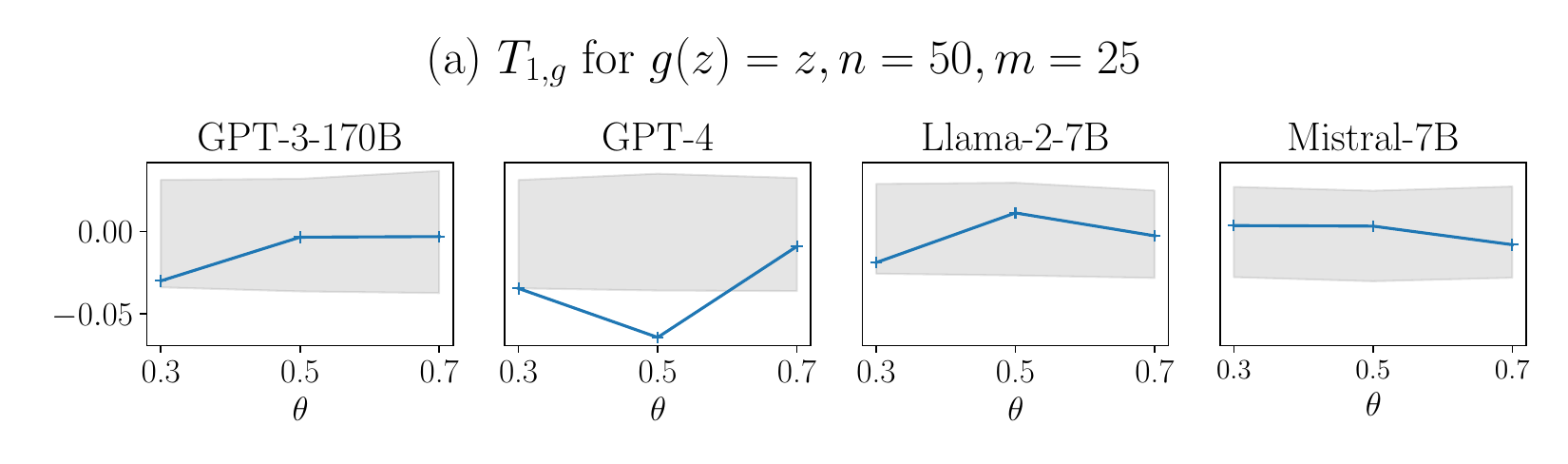}
    \hfill
    \includegraphics[width=0.495\textwidth,trim={0.1cm 0.2cm 0.8cm 0.5cm}]{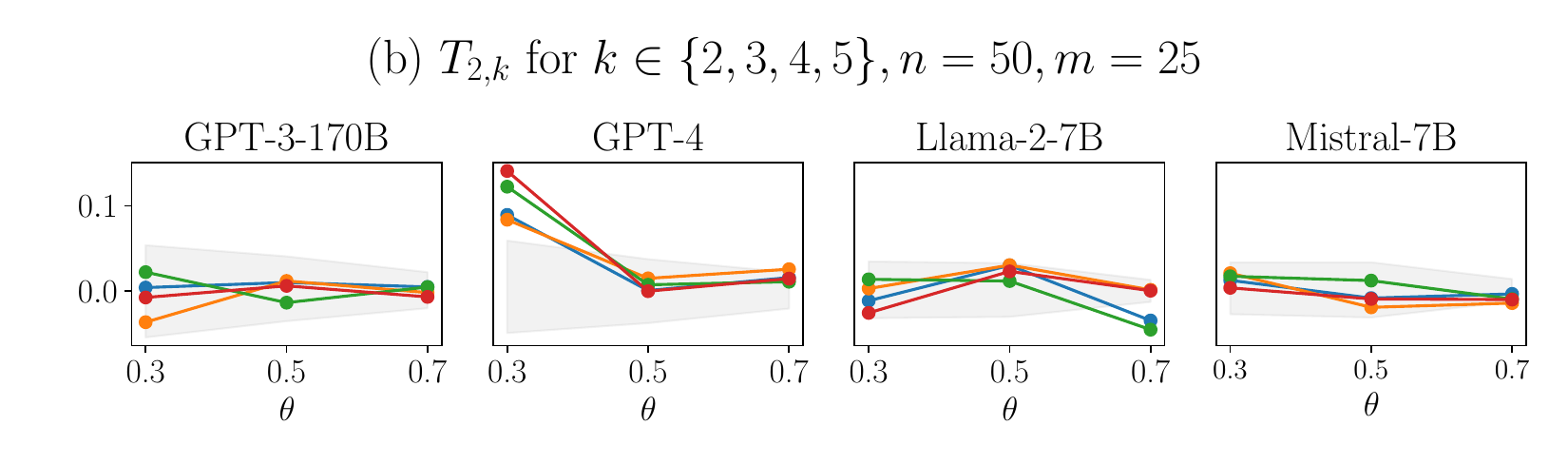}
    
    \includegraphics[width=0.49\textwidth,trim={0.9cm 0.2cm 0 0.5cm}]{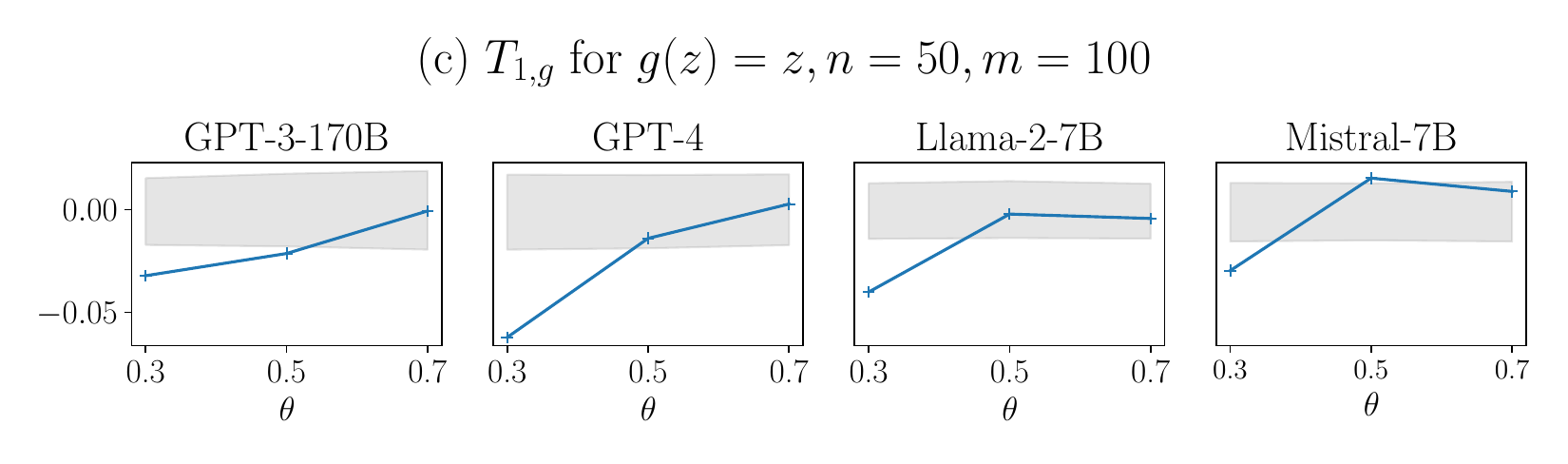}
    \hfill
    \includegraphics[width=0.49\textwidth,trim={0.1cm 0.2cm 0.8cm 0.5cm}]{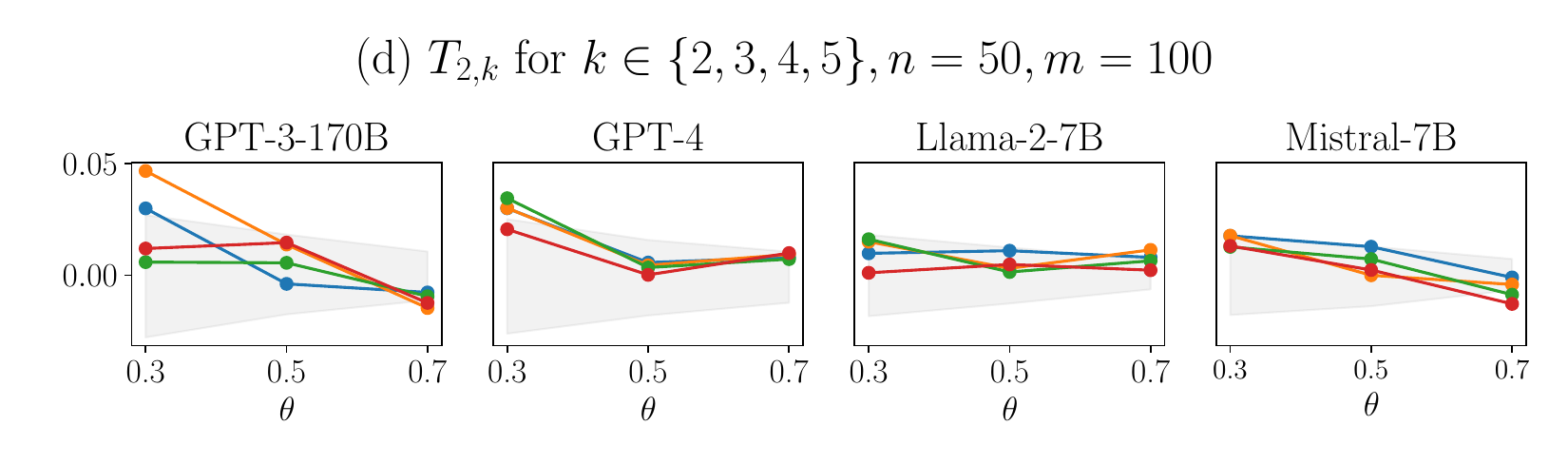}

    \caption{
    Checking the martingale property on %
    Bernoulli experiments.
    Each data point represents a test statistic (y-axis) %
    evaluated for an LLM, %
    as derived in %
    \S\ref{sec:CID-check}. Subplot and x-axis correspond to choices of Bernoulli probabilities and LLMs. %
    Shade indicates the $95\%$ confidence interval from a reference Bayesian model. %
    }  
    \label{fig:cid-bern-main}
\end{figure*}

As discussed in \S \ref{sec:The martingale property enables a principled notion of uncertainty}, we use sample paths generated by an LLM to approximate a martingale posterior and estimate its epistemic uncertainty. 
Here, we characterise epistemic uncertainty through the trace of the posterior covariance of the martingale posterior, the `spread' of the distribution.  %
Because the sample paths we use are finite (see \S \ref{sec:acceptable-deviations}) we cannot study the exact martingale posterior directly, which can only be recovered with infinite samples.  %
Instead, we study the sampling distribution of the maximum likelihood estimate (MLE) on the first $m$ samples:
$ %
\hat\theta_m := \arg\max_{\theta\in\Theta} \sum_{i=1}^{m} \log p_\theta(Z_{n+i}),
$
where $p_\theta$ is the known parametric likelihood. %
We measure the spread of this distribution using its \textit{inter-quartile range} 
\begin{equation}\label{eq:approxmp-iqr}
T_3 =    %
Q_{0.75}(\{\hat\theta_m^{(j)}\}_{j=1}^J) - 
Q_{0.25}(\{\hat\theta_m^{(j)}\}_{j=1}^J),
\end{equation}
where $\hat\theta_m^{(j)}$ denotes the MLE using the $j$-th sample path $\{z_{n+i}^{(j)}\}_{i=1}^m$, and $Q_{0.25}$ and $Q_{0.75}$ are the $0.25$- and $0.75$-quantiles. 
In our experiments in \S \ref{sec:Experiments} we consider scenarios where the true data distribution is defined by regular parametric models. %
In such cases the optimal (squared) estimation error for the true parameter 
scales $O(d/n)$ where $n$ is the ICL dataset size and $d$ is the dimension of the parameter, which is also the minimax lower bound %
\citep[Ch.~8]{van2000asymptotic}.
When choosing $m=\Theta(n)$, a reference Bayesian model will also have the $O(d/n)$ scaling behaviour following classical posterior contraction results in statistics; see App.~\ref{app:MP-approx}.
Therefore, we can compare the asymptotic scaling of $T_3$ %
between an LLM and a reference Bayesian parametric model through the same sampling-based procedure. %
If the scaling behaviour of $T_3$ from our LLM %
deviates from that %
of the reference Bayesian model, we can conclude that the LLM either exhibits a marked loss of estimation efficiency, %
or does not maintain a correct notion of epistemic uncertainty at all. 
Both characteristics contradict a Bayesian ICL system and are undesirable.

\section{Experimental Analysis on LLMs}  %
\label{sec:Experiments}

In this section, we experimentally probe whether ICL in state-of-the-art LLMs is Bayesian using the diagnostics discussed in \S \ref{sec:sec3} and corresponding test statistics $T_{1,g}, T_{2,k}, T_3$.
We provide our code base on \href{https://github.com/meta-inf/bayes_icl}{\url{https://github.com/meta-inf/bayes_icl}}.

\subsection{Experiment Setup}
\label{sec:exp-setup}

We consider three types of synthetic datasets $z_{1:n}$: %
\begin{itemize}[leftmargin=*,noitemsep, topsep=0pt, partopsep=0pt]
\item \textbf{Bernoulli}: $Z_i\sim \mathrm{Bern}(\theta)$, where $\theta\in\{0.3, 0.5, 0.7\}$;
\item \textbf{Gaussian}: $Z_i\sim \mathcal{N}(\theta, 1)$, where $\theta\in\{-1, 0, 1\}$;
\item A synthetic \textbf{natural language} experiment representing a prototypical clinical diagnostic task, where $Z_i=(X_i,Y_i)$ indicate the presence or absence of a symptom and disease as a text string for the $i$-th patient, respectively. %
Further, %
$
X_i \sim \mathrm{Bern}(0.5), Y_i\vert X_i\sim \mathrm{Bern}(0.3 + 0.4 X_i). 
$
\end{itemize}
On purpose, we reduce our experimental setup to these minimum viable test beds where the ground-truth latent parameters are known, stripping away the convoluted latent complexity of in-the-wild NLP data.  %
We use the following LLMs:
\href{https://huggingface.co/meta-llama/Llama-2-7b-hf}{\texttt{llama-2-7B}}
with 7B parameters (\citealp{touvron2023llama}), 
\href{https://huggingface.co/mistralai/Mistral-7B-v0.1}{\texttt{mistral-7B}}
(\citealp{jiang2023mistral}), 
\texttt{gpt-3} \citep{brown2020language} with 2.7B and 170B parameters, 
\texttt{gpt-3.5}, 
and \texttt{gpt-4} \citep{openai2023gpt4}\footnote{
We only use \texttt{gpt-4} in a subset of experiments (Fig.~\ref{fig:cid-bern-main}, Fig.~\ref{fig:icl-both} in the text) 
due to API and resource limitations (App.~\ref{app:additional experiment details}).  
}.

In all experiments we compute %
test statistics on LLM samples, and compare their behaviour with the same statistics evaluated on samples from a reference Bayesian model. 
More specifically, in \S\ref{sec:exp-cid} we compare the statistics obtained from LLMs 
with the bootstrap confidence intervals (CIs) derived from the reference Bayesian model. 
A deviation will thus indicate that 
the LLM is unlikely to be a good approximation of the reference Bayesian model. %
More importantly, when $n$ becomes moderately large, 
the Bernsten von-Mises theorem \citep{van2000asymptotic} applies:  
the deviations then %
imply that the LLM is highly likely deviating from all \textit{reasonable Bayesian models}, namely those satisfying the %
regularity conditions of the theorem. 
This is because the theorem guarantees that the test statistics derived from all such models have asymptotically\footnote{
We note that the asymptotic equivalence results are relevant in our setting. As a concrete example, in the setting of Fig.~\ref{fig:cid-bern-main}~(a), the CIs obtained by using $\mathrm{Beta}(1,11)$ and $\mathrm{Beta}(1,1)$ as the reference model are practically indistinguishable; the difference is on the order of $10^{-4}$. 
} equivalent distributions.

We refer to App.~\ref{app:additional experiment details} for additional experimental details, such as the prompt format, tokenization, and computational requirements, as well as additional experimental results.

\subsection{Checking the Martingale Property} \label{sec:exp-cid}

We first check if state-of-the-art LLMs satisfy the martingale property.   %
As we discussed in \S\ref{sec:Background}, this is a necessary condition for an exchangeable Bayesian ICL system.

\paragraph{Bernoulli experiment. } 
Fig.~\ref{fig:cid-bern-main} reports the results of the Bernoulli experiments with $n=50$ observed samples, LLM sample paths of length $m\in\{n/2, 2n\}$, %
and datasets with ground-truth mean $\theta \in \{.3,.5,.7\}$. 
As discussed in \S\ref{sec:CID-check} and \S\ref{sec:exp-setup} above, we compute the test statistics $T_{1,g}$ and $T_{2,k}$ on %
$J$ sample paths %
generated by an LLM, and compare them with %
bootstrap CIs (of high confidence, see scale of y-axis)  %
obtained from a reference Bayesian model. Here we define the reference model using a Bernoulli likelihood and a non-informative $\mathrm{Beta}(1,1)$ prior. 

For short sample paths of length $m=n/2$ 
(subplots (a) and (b)), most LLMs lead to test statistics that are generally within the %
respective CIs, with the main exception being \texttt{gpt-4} ($\theta\in\{0.3,0.5\}$), 
indicating a mostly adherence to the martingale property. %
However, for longer sample paths with $m=2n$ (subplots (c) and (d)), %
more frequent deviations from the CIs are observed. 
For brevity, full 
results for other choices of $n$ and LLMs 
are deferred to App.~\ref{app:further-experimental-results}. 
The findings are generally consistent across all choices of $n$. 
We also observe \texttt{gpt-3.5} to perform better than \texttt{gpt-4} but worse than \texttt{gpt-3-170b}. 
As we discuss in App.~\ref{app:further-experimental-results} the latter observation  
may be explained by the fact that \texttt{gpt-3.5} and \texttt{gpt-4} have undergone instruction tuning %
\citep{ouyang2022training}. 
In summary, in the Bernoulli experiments the LLMs generally adhere to the martingale property in short sampling horizons, but in longer horizons demonstrate a significant deviation from the martingale property and hence the Bayesian principle.

\paragraph{Gaussian experiment. } 
In Fig.~\ref{fig:cid-gauss-main} we present results on the Gaussian experiment with $\theta=-1,n=100, m=n/2$, again performing both checks of the martingale property and using a reference Bayesian model with the non-informative prior $\mathcal{N}(0, 100)$.
As we can see, all models except \texttt{gpt-3.5} 
demonstrate clear deviation from the martingale property. %
Additional results for \texttt{gpt-3.5} in App.~\ref{app:Additional experimental details and results} present our diagnostics with other choices of $(n,m,\theta)$, demonstrating a deviation from the predictive distribution %
of the reference Bayesian posterior. %
In conclusion, the presented evidence on the Gaussian experiment falsifies 
our hypothesis of Bayesian behaviour with the tested LLMs.

\begin{figure}[ht]
    \centering
    \includegraphics[width=.95\linewidth,trim={1.5cm 0.5cm 1.5cm 0.5cm}]{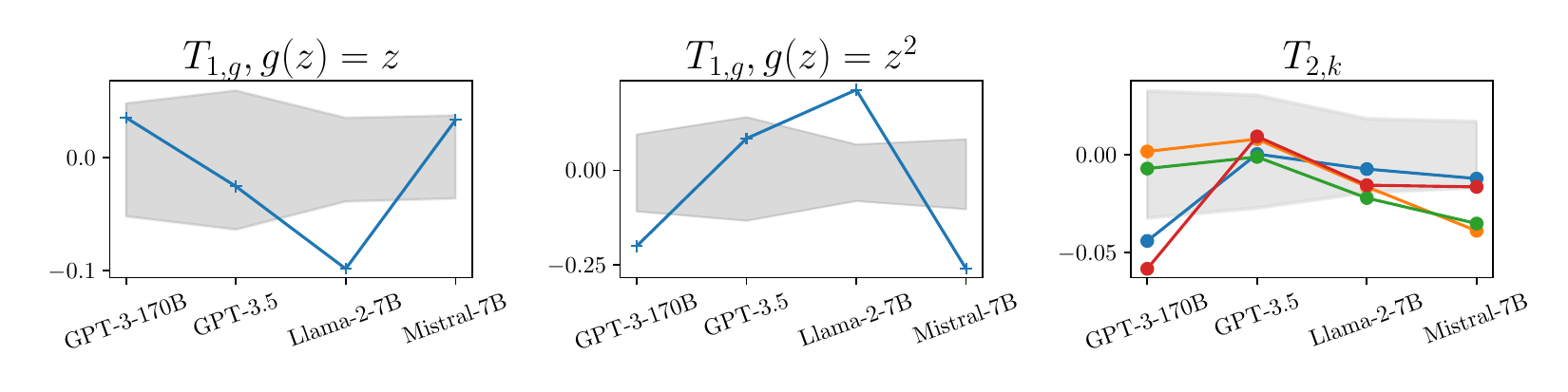}
    \caption{
    Checking the martingale property %
    on Gaussian experiments.
    We present runs with $\theta=-1,n=100,m=50$ from different LLMs (x-axis) with test functions $g(z)=z$ and $g(z)=z^2$.  %
    See Fig.~\ref{fig:cid-bern-main} for further details.  %
    }%
    \label{fig:cid-gauss-main}
\end{figure}

\paragraph{Synthetic natural language experiment. } 
In %
Fig.~\ref{fig:icl-both} we present our results for the natural language experiment with $n=80, m=40, g(z)=z$ %
using the GPT models. %
Here, we compute the test statistics on samples separated by the Bernoulli-distributed value of $X_i$ (see App.~\ref{app:additional experiment details} for details). 
As we can see, %
both \texttt{gpt-3.5} and \texttt{gpt-4} %
demonstrate %
deviation from a reference Bayesian posterior. 
This provides further evidence of violations of the martingale property in settings where natural language (instead of numbers) 
is used.   %

\begin{figure}[ht]
\centering
\subfigure[GPT-3-170B]{
\includegraphics[width=0.32\linewidth,clip,trim={0.2cm 0.4cm 0.2cm 0.2cm}]{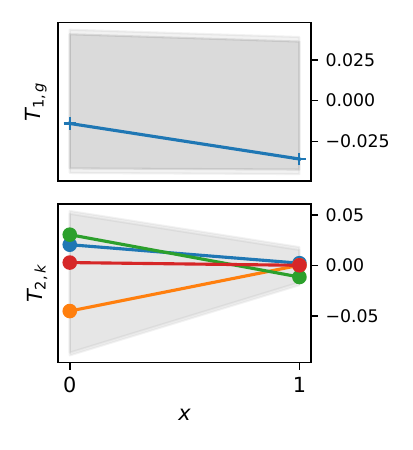}}\hfill
\subfigure[GPT-3.5]{
\includegraphics[width=0.32\linewidth,clip,trim={0.2cm 0.4cm 0.2cm 0.2cm}]{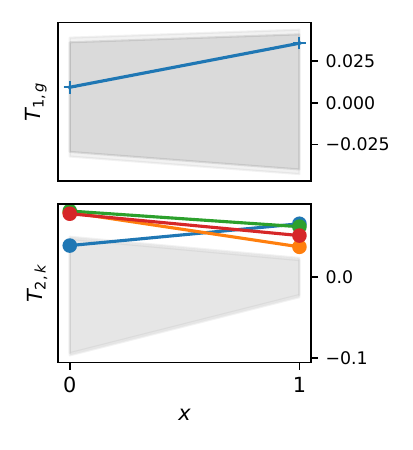}}\hfill
\subfigure[GPT-4]{
\includegraphics[width=0.32\linewidth,clip,trim={0.2cm 0.4cm 0.2cm 0.2cm}]{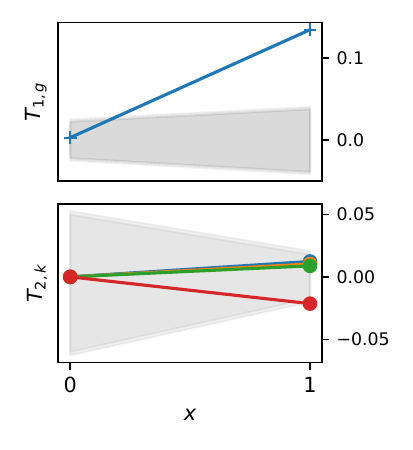}}
\caption{
    Checking the martingale property %
    on the natural language experiment.
    We present both checks %
    with test statistics %
    computed separately for each value of $X_i$ (x-axis). %
    See Fig.~\ref{fig:cid-bern-main} for further details.
}\label{fig:icl-both}
\end{figure}

\begin{figure}[ht]
\centering
\includegraphics[width=0.95\linewidth,clip,trim={0 0.6cm 0 0.2cm}]{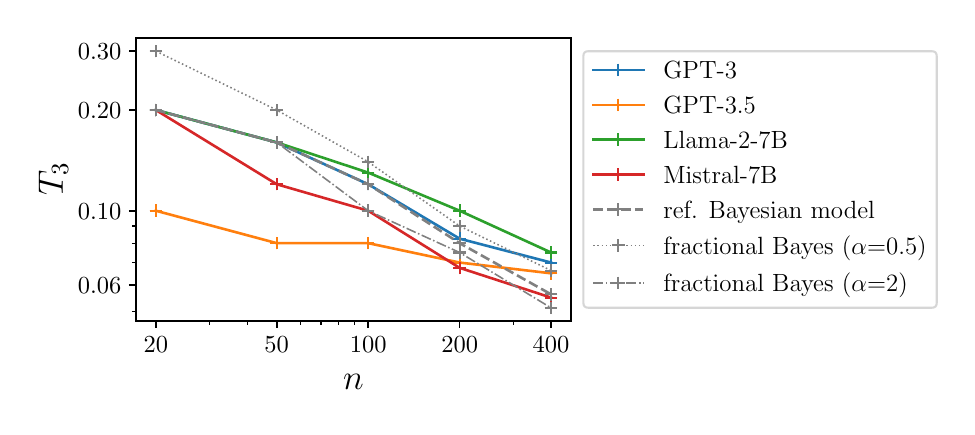}
\caption{
    Scaling of epistemic uncertainty on the Bernoulli experiment: the test %
    statistic $T_3$ %
    (\S \ref{sec:additional-checks}) computed on LLMs, %
    compared with 
    Bayesian %
    and fractional Bayesian models.  %
}\label{fig:scaling-epis}
\end{figure}

\subsection{Checking Epistemic Uncertainty of LLMs}
\label{sec:exp_epistemic}

In this subsection we analyse the scaling behaviour of an LLM's uncertainty.
In %
Fig.~\ref{fig:scaling-epis} 
we measure $T_3$ (y-axis on a log-scale) and compare the approximate martingale posterior of an LLM with a reference Bayesian model when increasing the number of observed samples $n$ (x-axis).
We consider a Bernoulli experiment with $\theta=0.5$ as it is the only experimental setting where, with a short sampling horizon of %
$m=n/2$, all LLMs approximately adhere to the martingale property.
In addition to the standard reference Bayesian model, we also consider two $\alpha$-fractional Bayesian posteriors \citep{bhattacharya2019bayesian}, which are generalisations of %
the Bayesian posterior that exhibit a $O(d/\alpha n)$ scaling for its epistemic uncertainty.
They allow us to check the weaker hypothesis whether an LLM's epistemic uncertainty scales at least up to the correct order of magnitude.

We observe that the asymptotic rate of \texttt{llama-2-7b} and \texttt{gpt-3.5} is slower than that of a Bayesian model, which suggests %
inefficiency as discussed in \S \ref{sec:additional-checks}.
Furthermore, \texttt{gpt-3.5} demonstrates over-confidence in the small-sample %
regime. 
The scaling of 
\texttt{gpt-3-170b} and \texttt{mistral-7b} are closer to the Bayesian model, even though not exactly matching the latter. 
This finding is interesting as %
on the Bernoulli experiments, \texttt{gpt-3-170b} and 
\texttt{mistral-7b} also demonstrate the best adherence to the martingale property.

\section{Conclusion}

In this work we stated the martingale property as a fundamental requirement of a Bayesian learning system for exchangeable data, and discussed its desirable consequences if satisfied by an LLM. Based on this property we derived three different diagnostics that allowed us to check whether LLMs adhere to the Bayesian principle on synthetic in-context learning tasks. We presented stark evidence that state-of-the-art LLMs violate the martingale property, and hence falsified the hypthesis that ICL in LLMs is Bayesian.

Our investigation is particularly relevant to a recent line of work that investigates LLM-based ICL for tabular data modelling: for prediction on noisy tabular datasets \citep{manikandan_language_2023,yan_making_2024}, the martingale property would enable us to diagnose the predictive uncertainty; and for synthetic data generation \citep{borisov2022language,hamalainen2023evaluating,veselovsky2023generating}, 
it is vital to ensuring valid inference based on imputations of missing data (\S\ref{sec:martingale}). %
It is thus of practical interest to develop models that better adhere to the martingale property. 

The primary limitation of our work is the (intentional) restriction to small-scale, synthetic datasets, which are different from common NLP applications. 
We note that while our diagnostics are designed for synthetic problems, 
they reflect a broader principle: Bayesian epistemic uncertainty %
can be extracted from black-box models by examining the correlation structure in sequential predictions. 
This is clearly shown by the variance estimator in Example~\ref{ex:corr-1-and-mean-param}, and by the fact that MLE on sampled paths approximates the Bayesian posterior (\S\ref{sec:additional-checks}). 
Future work could investigate generalisations of this approach.

More broadly, the RCT example in \S\ref{sec:Introduction} %
can arguably be viewed as the simplest type of decision task involving multi-step reasoning, as the right decision (here based on an average treatment effect) is only naturally determined after imputing all missing samples. 
Thus, it would be interesting to investigate analogies to the hallucination behaviour we have identified for ICL in more complex reasoning tasks such as those involving chain-of-thought prompting \citep{wei2022chain}. 
Lastly, it may be worth to consider fine-tuning objectives to achieve an idealised Bayesian behaviour with a model after pretraining, but before deployment.

\section*{Impact Statement} %

This paper presents work whose goal is to advance the field of Machine Learning. There are many potential societal consequences of our work, none which we feel must be specifically highlighted here.
We refer to App. \ref{app:Negative societal impact} for further discussion.

\section*{Acknowledgments}

Fabian Falck acknowledges the receipt of studentship awards from the Health Data Research UK-The Alan Turing Institute Wellcome PhD Programme (Grant Ref: 218529/Z/19/Z).
Ziyu Wang acknowledges support from Novo Nordisk.
Chris Holmes acknowledges support from the Medical Research Council Programme Leaders award MC\_UP\_A390\_1107, The Alan Turing Institute, Health Data Research, U.K., and the U.K. Engineering and Physical Sciences Research Council through the Bayes4Health programme grant.

This research is supported by research compute from the Baskerville Tier 2 HPC service. 
Baskerville is funded by the EPSRC and UKRI through the World Class Labs scheme (EP/T022221/1) and the Digital Research Infrastructure programme (EP/W032244/1) and is operated by Advanced Research Computing at the University of Birmingham.
We further acknowledge the receipt of OpenAI API credits through the OpenAI Researcher Access Program.

\bibliography{8_bib.bib}
\bibliographystyle{icml2024}

\clearpage
\clearpage
\appendix

{\textbf{\Large Appendix for \textit{Is In-Context Learning in Large Language Models Bayesian? A Martingale Perspective}}}

\section{Proofs of Theoretical Statements in the Main Text}
\label{app:proof}

\begin{fact}
Any exchangeable random sequence $\{Z_i\}$  must be conditionally identically distributed. %
\end{fact}
\begin{proof}
See, e.g., \citet[p.~2030]{Berti2004}.
\end{proof}

\prop*
\begin{proof}
It suffices to show the equivalence between the following three statements:
\begin{enumerate}[label=(\roman*)]
\item $Z_{n+1:n+m}\vert Z_{1:n}$ satisfies Eq.~\eqref{eq:martingale} 
\item for all $n'\ge n,k\ge 1$ and integrable function $g$ we have 
$\Eb(g(Z_{n'+k})-g(Z_{n'+1})\vert Z_{1:n}, Z_{n+1:n'})=0$ 
\item 
for all $n'\ge n,k\ge 1$ and integrable $(g, h)$ we have 
$0 %
= \Eb((g(Z_{n'+k})-g(Z_{n'+1}))h(Z_{n+1:n'})\vert Z_{1:n}).
$
\end{enumerate}
The equivalence between (i) and (ii) is trivial.
We have (ii) $\Rightarrow$ (iii) because $
\Eb((g(Z_{n'+k}) - g(Z_{n'+1}))h(Z_{n+1:n'})\mid Z_{1:n}) = 
\Eb(\Eb(g(Z_{n'+k}) - g(Z_{n'+1})\mid Z_{1:n'})h(Z_{n+1:n'})\mid Z_{1:n}) \overset{(ii)}{=} 0.$
To show (iii) $\Rightarrow$ (ii), for any $\sigma(Z_{n+1:n'})$-measurable set $A$ %
let $h := \mathds{1}_A$ be the respective indicator function, so that %
$\Eb((g(Z_{n'+k}) - g(Z_{n'+1})) \mathds{1}_A \mid Z_{1:n}) \overset{(iii)}{=} 0 = \Eb(0\cdot \mathds{1}_A \mid Z_{1:n})$. 
Since this holds for all $A$, it follows by the definition of conditional expectation \citep{kallenberg1997foundations} that 
$\Eb(g(Z_{n'+k}) - g(Z_{n'+1})\mid Z_{1:n'}) = 0$ a.s..
\end{proof}

\martingaleimplications*
\begin{proof}
(i) follows by setting $h(z_{n+1:n'}) \equiv \mathds{1}$ in \eqref{eq:CID-equiv-form}. (ii) follows by setting $g(z) =z, h(z_{n+1:n'})=z_{n'}$. 
\end{proof}

\posteriorvarianceesterror*
\begin{proof}
    This holds because $\theta$ and $\theta_0$ are conditionally independent and identically distributed given $z_{1:n}$, and $\bar\theta_n$ equals the conditional expectation of both random variables. 
\end{proof}

\section{Further Discussion of Theory and Methodology}
\label{app:Additional theoretical details}

\subsection{Additional Background on Martingale Posteriors}\label{app:mp} 

In \S\ref{sec:The martingale property enables a principled notion of uncertainty} we discussed the construction of martingale posteriors in the finite-support case. 
Here, we can construct the martingale posterior by sampling $Z_{n+1:n+m}\vert Z_{1:n}$, which will determine a sample $\theta_{n+m}\vert Z_{1:n}$ as the parameter that indexes the predictive distribution 
$p(Z_{n+m+1}=\cdot\vert Z_{1:n+m}) = p_{\theta_{n+m}}(\cdot)$; and since $\theta_{n+m}\to\theta_\infty$ as $m\to\infty$, we can truncate the process at a large $m\gg n$ to obtain a good approximation for $\theta_\infty$. 

The restriction to finite support is largely for expository simplicity as it allows us to avoid measure-theoretic considerations. 
More generally, it is always possible to view the distribution $p(Z_{n+1}=\cdot | Z_{1:n}) =: \theta_n$ as a random element in a suitable Banach space of measures and the condition in Eq. \eqref{eq:martingale} as requiring 
$\{p(Z_{n+1}=\cdot | Z_{1:n}): n\in\mb{N}\}$ to define a martingale in that space. 
When Doob's theorem applies, the above construction provides a distribution over predictive distributions that quantifies the epistemic uncertainty.

Nonetheless, for tractability and comparability to %
Bayesian parametric posteriors, it is useful to consider the following alternative, `model-based' procedure:
\begin{enumerate}
    \item Sample $Z_{n+1:n+m}\sim p_M(\cdot\vert Z_{1:n})$.
    \item Compute $\hat\theta_m := \arg\max_{\theta\in\Theta} \sum_{j=1}^m \log p(Z_{n+j}\vert\theta)$.
    \item Return $\hat\theta_m$ as an approximate sample from the martingale posterior, defined as the conditional distribution of the pointwise limit $\lim_{m\to\infty}\hat\theta_m$ given $Z_{1:n}$. 
\end{enumerate}
We repeat this procedure to obtain multiple samples $\hat\theta_m$ from the martingale posterior in order to approximate its distribution (see Fig. \ref{fig:figure1} [Centre]).
In the above, $p(Z_i\vert\theta)$ is the likelihood in the Bayesian parametric model. If $\{p_M(Z_{n+j}\vert Z_{1:n+j-1})\}_{j=1}^\infty$ corresponds to a certain posterior predictive defined by the same likelihood, and the model is such that maximum likelihood estimation is consistent, it follows from de Finetti's theorem (applied to $Z_{n+1:}\vert Z_{1:n}$) and consistency that as $m\to\infty$, $\hat\theta_m$ will converge to a random variable $\hat\theta_\infty$ (w.r.t.~the norm and notion of convergence in consistency), and  
the distribution $\hat\theta_\infty\vert Z_{1:n}$ must equal the Bayesian posterior. Applying the same procedure to a more general $p_M$ that satisfies Eq. \eqref{eq:martingale} leads to the methodology in \citet{fong2021martingale}. 

We adopted this `model-based' approach in \S\ref{sec:additional-checks} and for computing the approximate martingale posterior in Fig. \ref{fig:figure1} [centre]. Compared with the former approach, it is easier to implement on ICL tasks where each sample $Z_i$ is represented with multiple tokens and a correctly specified likelihood for the true observations is available; the latter is always true in our synthetic experiments.
More importantly, when $m$ is finite (and not $\gg n$), only with this approach can we compare the sampling distribution of $\hat\theta_m\vert Z_{1:n}$ across different $p_M$, as we explain in the following. This is important in our experiments where we find the LLMs (at best) follow the martingale property within a horizon of $m = \Theta(n)$. 

\subsection{Approximate Martingale Posteriors with Finite Paths}\label{app:MP-approx}

We have claimed that with a finite $m$, the spread of the approximate martingale posterior $\hat\theta_m$ defined as the MLE on $m$ samples (see \S\ref{sec:additional-checks}, or above%
) is comparable between different choices of $p_M$. We now substantiate on this claim. 

Let us first restrict to %
exchangeable (i.e., Bayesian) choices of $p_M$. 
Consider de Finetti's representation for the posterior predictive measure: $Z_{n+1,\ldots}\vert Z_{1:n}$ can be represented through 
$$
\theta_\infty\sim \pi(\cdot\vert Z_{1:n}), ~~ Z_{n+1,\ldots}\overset{iid}{\sim} p(\cdot\vert \theta_\infty)
$$
where the measure $\pi(\cdot\vert Z_{1:n})$ equals the Bayesian posterior, which as discussed in \S\ref{app:mp} equals the exact martingale posterior. Combining the above representation and the fact that $\hat\theta_m$ is a function of $Z_{n+1:n+m}$ leads to $\hat\theta_m\perp Z_{1:n}\vert \theta_\infty$, and %
\begin{align*}
&\phantom{=}\mathrm{Cov}(\hat\theta_m\vert Z_{1:n}) \\
&= 
    \Eb(\mathrm{Cov}(\hat\theta_m\vert \theta_\infty%
    )\vert Z_{1:n}) + 
    \mathrm{Cov}(\Eb(\hat\theta_m\vert \theta_\infty%
    )\vert Z_{1:n}) \\ 
&\approx 
    \Eb(\mathrm{Cov}(\hat\theta_m\vert \theta_\infty%
    )\vert Z_{1:n}) + 
    \mathrm{Cov}(\theta_\infty%
    \vert Z_{1:n}),
\end{align*}
where we dropped the term $\Eb(\hat\theta_m\vert\theta_\infty) - \theta_\infty$ which is the %
bias of MLE and thus a higher-order term for regular models. 
Therefore, the (co)variance overhead $\mathrm{Cov}(\hat\theta_m\vert Z_{1:n}) - \mathrm{Cov}(\theta_\infty\vert Z_{1:n})$ is, 
up to the first order, the average-case error of MLE on $m$ i.i.d.~samples when the true parameter is sampled from the posterior $\pi(\cdot\vert Z_{1:n})$. 
For regular models this is always $\Theta(d/m)$, where the coefficient hidden in the $\Theta$ notation is also comparable across different $p_M$ as long as the Fisher information matrix evaluated at $\theta\sim\pi(\cdot\vert Z_{1:n})$ has a comparable value (e.g., across all choices of $p_M$ that satisfy \textit{consistency}). 
As the martingale posterior covariance $\mathrm{Cov}(\theta_\infty\vert Z_{1:n})$ 
has the same 
$\Theta(d/n)$ scaling across all regular Bayesian models to which the Bernstein von-Mises theorem applies, with a choice of $m = \Theta(n)$, any deviation in the scaling of $\mathrm{Cov}(\hat\theta_m)$---from that of any regular Bayesian model---must be attributable to a different scaling of the exact MP covariance, and thus a deviation from all regular Bayesian models. %

Lastly, we note that 
while we focus on ICL models that are approximately Bayesian, 
the above discussion may also apply to general models that only satisfy the martingale property, since for those models $Z_{n+1,\ldots}\vert Z_{1:n}$ remains asymptotically exchangeable \citep{Berti2004}. 
Moreover, the above discussion applies to inter-quantile range (IQR) as well, because for %
asymptotically normal posteriors the IQR is proportional to the posterior standard deviation; and even for non-normal posteriors, the IQR should still have the same order as the posterior contraction rate by definition. 

\subsection{Acceptable Approximation Errors of Properties (i) and (ii) in \cref{prop:martingale_implications}}\label{app:acceptable-dev}

Even when we restrict to a finite horizon $m$, there can still be expected deviations from Eq. \eqref{eq:martingale}, and thus those in  \cref{prop:martingale_implications}, simply because 
Eq. \eqref{eq:martingale} represents invariance conditions that are not ``hard-wired'' in the LLM's architecture. 
Yet, small violations of %
these equalities should not have practical consequences. 
We now derive the order of what is an acceptable violation %
in the setting of \cref{ex:corr-1-and-mean-param}. 

As discussed in this example, the equalities in \cref{prop:martingale_implications} 
guarantee the expressions for posterior mean and covariance for the %
parameter $\theta$ to have consistently defined values, regardless of the choices of $(n',k)$. 
The posterior mean has the order of $\Theta(1)$ and requires the violation of \Cref{prop:martingale_implications}~(i) to be $o(1)$. The posterior covariance is generally $\Omega(1/n)$ and can be expressed through \Cref{ex:corr-1-and-mean-param} %
as
\begin{align*}
\mathrm{Cov}(\theta\vert Z_{1:n}) &= \Eb(Z_{n+1}Z_{n+k}\vert Z_{1:n}) %
- \Eb(Z_{n+k}\vert Z_{1:n})^2. %
\end{align*}
Therefore, it can have an approximately consistent value %
if 
the equalities in \Cref{prop:martingale_implications} hold approximately \textit{up to an error of $o(1/n)$}. 
Posterior mean and covariance are key quantities in the interpretation of %
predictive uncertainty, which in turn is a major benefit of the martingale property. 
Thus, we consider the above deviation to be acceptable as it already guarantees the approximately consistent interpretation of predictive uncertainty through the martingale property. %

\section{Additional Experimental Details and Results}
\label{app:Additional experimental details and results}

\subsection{Additional Experimental Details}\label{app:additional experiment details}

\paragraph{Test statistics of properties implied by the martingale property. }
We summarise and empirically measure properties (i) and (ii) in Corollary \ref{prop:martingale_implications} using the aggregated statistics
\begin{align}
T_{1,g} &:= \frac{2}{Jm}\sum_{j=1}^J 
    \sum_{i=1}^{m/2} (g(z_{n+i}^{(j)}) - 
     g(z_{n+i+m/2}^{(j)})), \label{eq:check1-g} \\ 
T_{2,k} &:= \frac{1}{Jm} \sum_{j=1}^J\sum_{i=1}^{m-k-1} (z_{n+i+1}^{(j)} -  z_{n+i+k}^{(j)}) z_{n+i}^{(j)}. \label{eq:check2-k}
\end{align}
The statistics $T_{1,g}$ and $T_{2,k}$ are defined using samples $\{z_{n+i}^{(j)}\}$ from $J$ paths generated by an LLM via ICL and correspond to Monte-Carlo estimates of the expectations in properties (i) and (ii). 
To be robust against the possible outlier paths (\S\ref{sec:acceptable-deviations}), we remove sample paths with anomalous mean absolute values using the standard 1.5$\times$IQR rule. 

We compare the observed value of the statistics above evaluated on LLMs with bootstrap confidence intervals computed using a reference Bayesian model (\S\ref{sec:exp-setup}). For the latter, %
we draw $K=300$ sets of completions 
$\{\{z_{bs,n+i}^{(j,k)}: 1\le i\le m, 1\le j\le J\}: 1\le k\le K\}$ from the predictive distribution of the reference Bayesian model, which provides $K$ samples for the test statistics, and compute two-sided confidence intervals using the respective quantiles.

\paragraph{Experimental setup.}
For the first two experiments we vary $n\in\{20, 50, 100\}$, $m\in \{n/2, 2n\}$ and sample $J=200$ paths from the LLMs. 
For the natural language experiments we fix $n=100, m=50, J=80$.  
As non-exchangeable models may demonstrate different behaviour on different permutations of the same dataset, 
for the experiments in \S\ref{sec:exp-cid} we permute the observations when generating each sample path, so that we can produce a single test statistic that summarises each experiment configuration. For the experiments in \S\ref{sec:exp_epistemic}, however, we use a fixed ordering for the observations for all path samples within each run, and report the median inter-quartile range across 9 runs for each configuration. This change is made to avoid (possibly small) deviations from exchangeability from inflating the estimated spread of the posterior. 

For a proper test of the martingale property, it is vital that the model cannot distinguish between the ICL training data $Z_{1:n}$ and its own generations $\{Z_{n+i}\}$. This is trivially true if the LLM takes free-form text as inputs without additional annotation, %
as with \texttt{llama-2-7b}, \texttt{mistral-7b}, and \texttt{gpt-3.5} accessed through the \texttt{Completion} API from OpenAI. 
However, the \texttt{gpt-4} model is only accessible through a different API (\texttt{ChatCompletion}) which includes annotation for user input and model generation in the prompt. To ensure a proper implementation of the checks, we hence call the API $m$ times in generating each path sample. 
In each iteration we sample a single data point, and then append it to the user input part of the prompt. 
This is far less cost-efficient than our use of %
\texttt{gpt-3.5}. Therefore, we only include \texttt{gpt-4} for the Bernoulli experiment with $n\le 50$, and the natural language experiment. 

We discuss prompt design and format in detail below. Here we emphasise that across all tasks, the prompt always includes sufficient information about the true likelihood.

\paragraph{Prompt design and format.}
We use the following prompt format \texttt{<instruction> <observed data> <sampled data>}.  %
\texttt{<instruction>} describes the distribution (i.e. true likelihood) of the observed data and importantly states that the observed samples were drawn i.i.d., i.e. from exchangeable random variables.
\texttt{<observed data>} and \texttt{<sampled data>} lists the observed $z_{1:n}$, and sampled data $\hat{z}_{n+k}$ (if there exists any), respectively.
Samples are represented depending on the experiment: as \texttt{int} values as 1-digit characters (e.g. `1'), \texttt{float} values with 1-digit of precision (e.g. `2.2') or words for synthetic natural language.  %
As a sanity check, we also consider replacing integers with random words (e.g. `tiger' for `1', `hedgehog' for `0'), but did not notice important differences in the LLMs' behaviour.
Each sample is delineated by a separator (e.g. `;').  %

We present exemplary prompts for each dataset below:
\begin{itemize}[leftmargin=*]
    \item A Bernoulli experiment with $n=5$ and $m=2$: 
\textit{``Provided are independent, identically distributed tosses of a coin, which flips 1 with probability p where p is unknown: 1;0,0,1,0,0;1''}.
\item A Gaussian experiment with $n=2$ and $m=3$: 
\textit{``Provided are independent, identically distributed draws from a Gaussian, with fixed but unknown mean and unit variance: 1.1,0.8,1.3,1.0,0.9''}. 
\item The the natural language experiment: 
\textit{
``You will make predictions for a novel disease. The observed dataset contains
records for multiple subjects which are assumed to be independent and identically distributed. For each subject there are two binary variables, indicating fever and disease diagnosis, respectively. Output your prediction for the disease diagnosis of the next subject.\textbackslash n
Id: 0\textbackslash n
Fever: Y\textbackslash n
Diagnosis: N
\ldots''
}
\end{itemize}

Other work represents both \texttt{int} and \texttt{float} numbers as a space-separated string of digits with fixed precision, where each number is separated by a semi-colon.
This guarantees a per-digit tokenisation that was observed to be beneficial in the context of time series forecasting and further minimises the required number of tokens per number as the decimal point is redundant \cite{gruver2023large}.
We did not opt for this representation and corresponding tokenisation for two reasons: 
First, initial experiments with GPT-2 showed deteriorating sampling performance, where the model often hallucinated unrelated content.
Second, and related to the first point, this representation is somewhat `out-of-distribution' and probably unseen in the training distribution, which could limit and constrain any conclusions made in our experiments.
Note that because of the tokenisation, in \S \ref{sec:Experiments}, the Gaussian experiment is more difficult than the Bernoulli experiment (or any dataset with single-token samples) as the LLM is required to learn the correlation structure between consecutive tokens representing a real-valued number.

\begin{figure*}[p]
\centering 
\subfigure[$T_{1,g}$ for $g(z)=z, n=20, m=10$]{\includegraphics[width=0.48\linewidth]{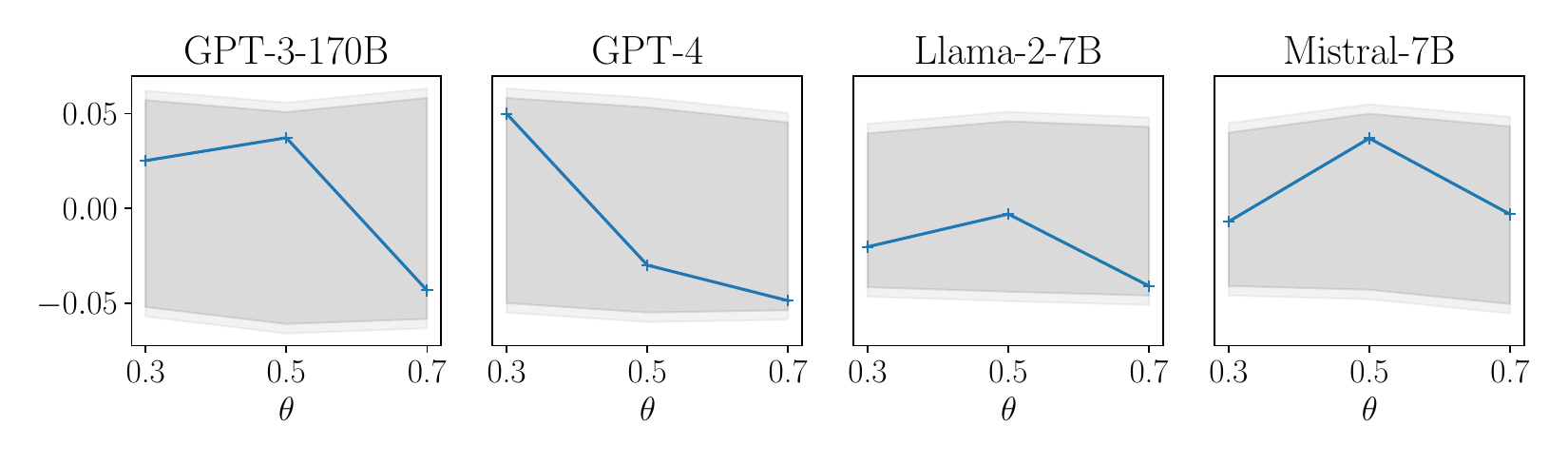}}
\subfigure[$T_{2,k}$ for $k\in\{2,3,4,5\}, n=20, m=10$]{\includegraphics[width=0.48\linewidth]{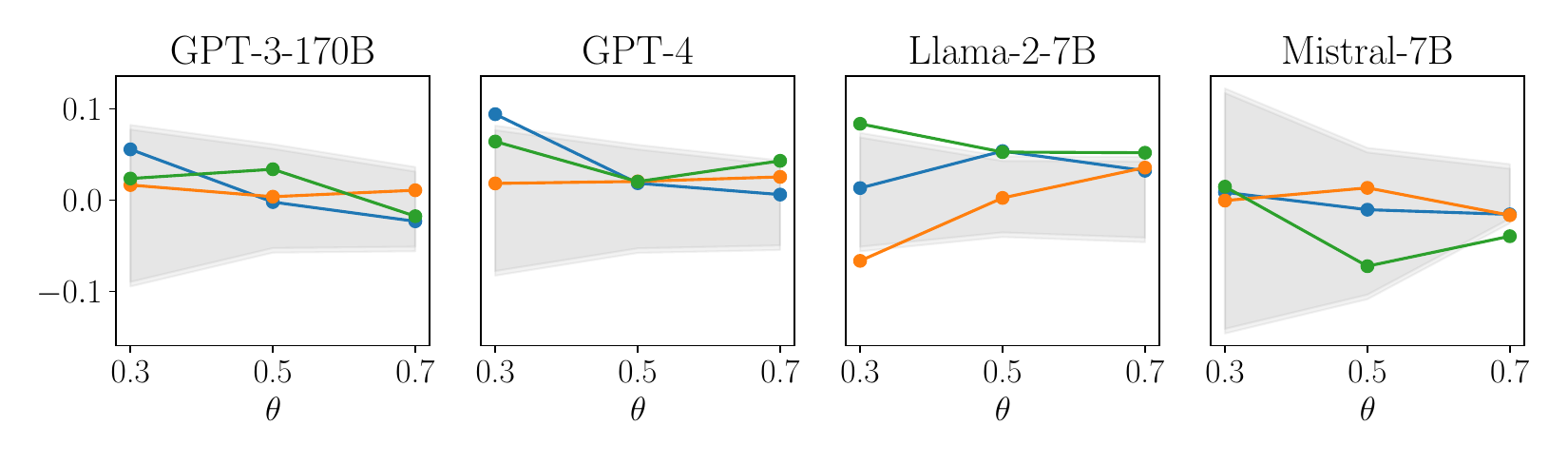}}
\subfigure[$T_{1,g}$ for $g(z)=z, n=100, m=50$]{\includegraphics[width=0.48\linewidth]{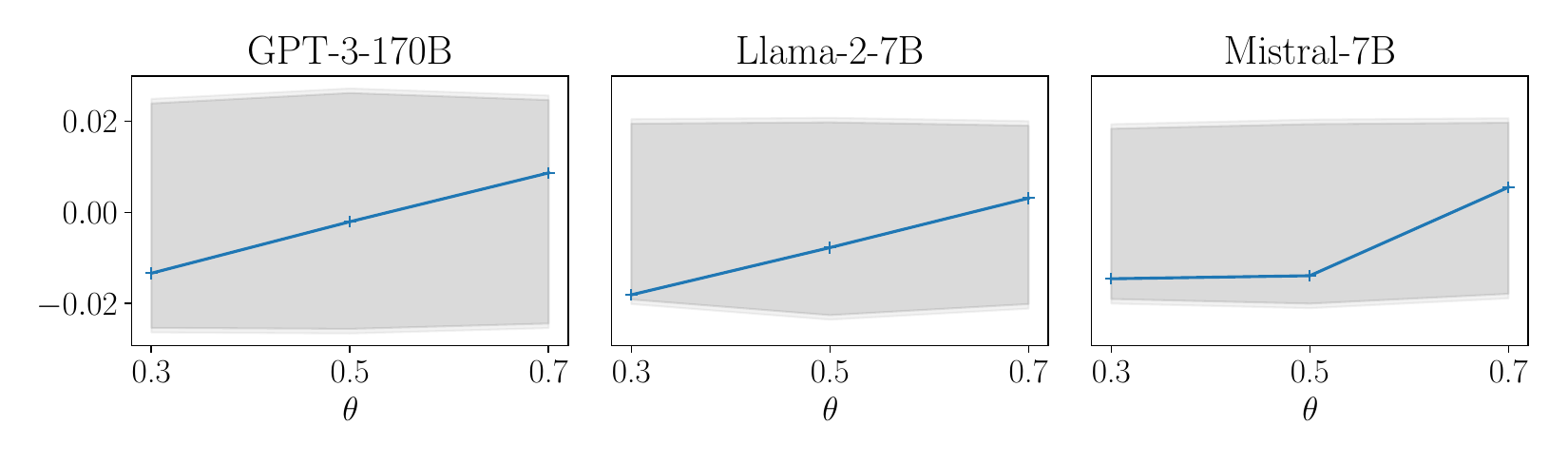}}
\subfigure[$T_{2,k}$ for $k\in\{2,3,4,5\}, n=100, m=50$]{\includegraphics[width=0.48\linewidth]{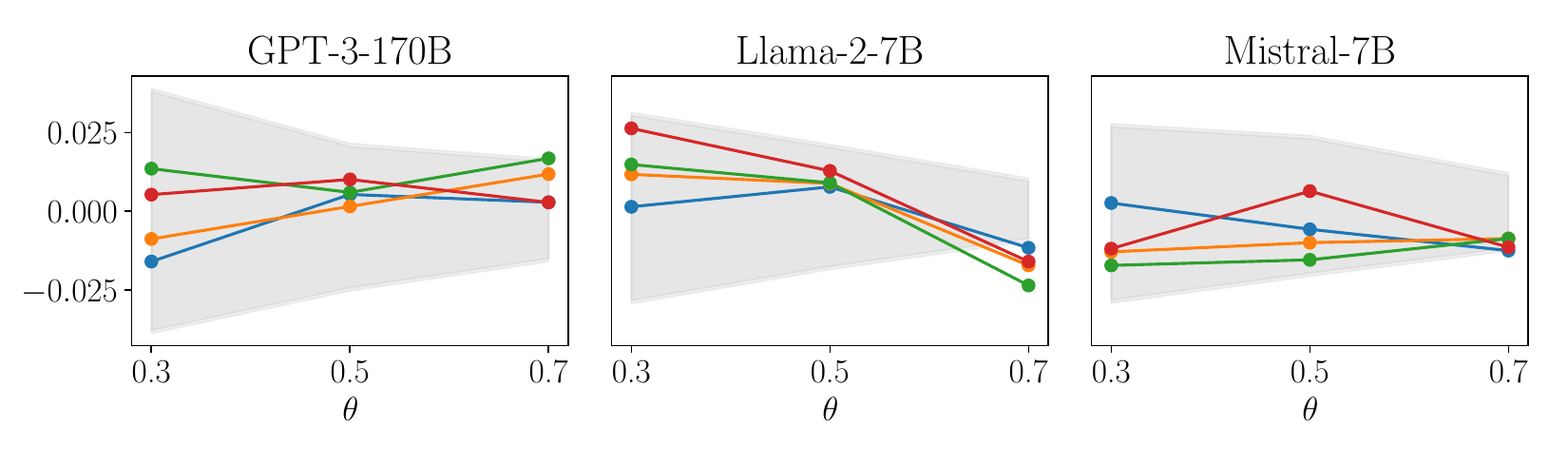}}
\subfigure[$T_{1,g}$ for $g(z)=z, n=20, m=40$]{\includegraphics[width=0.48\linewidth]{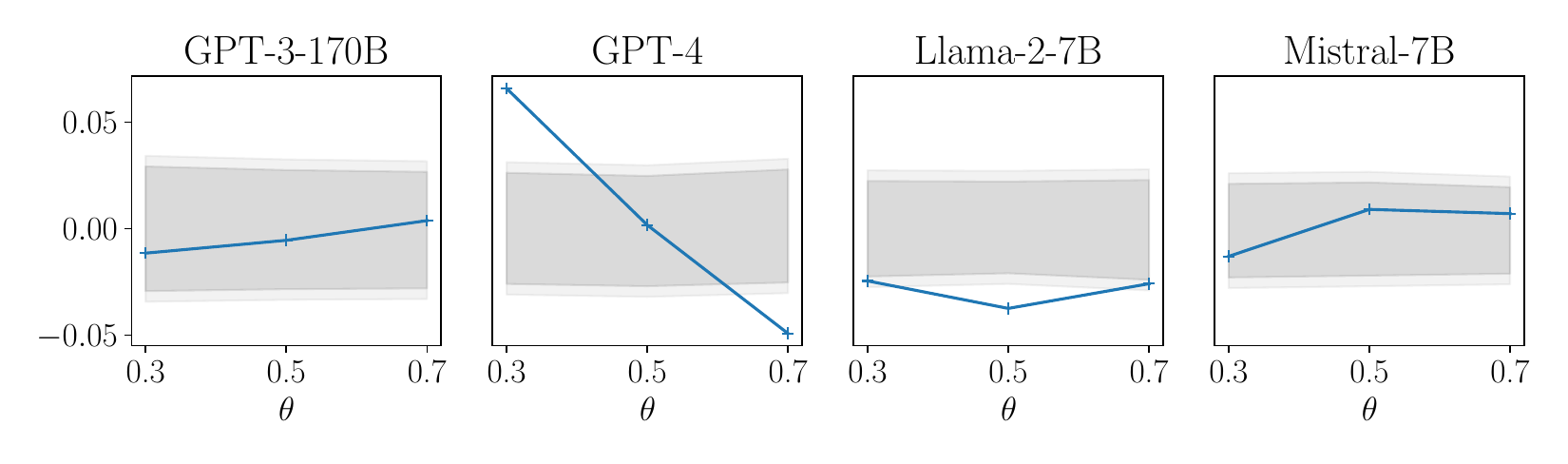}}
\subfigure[$T_{2,k}$ for $k\in\{2,3,4,5\}, n=20, m=40$]{\includegraphics[width=0.48\linewidth]{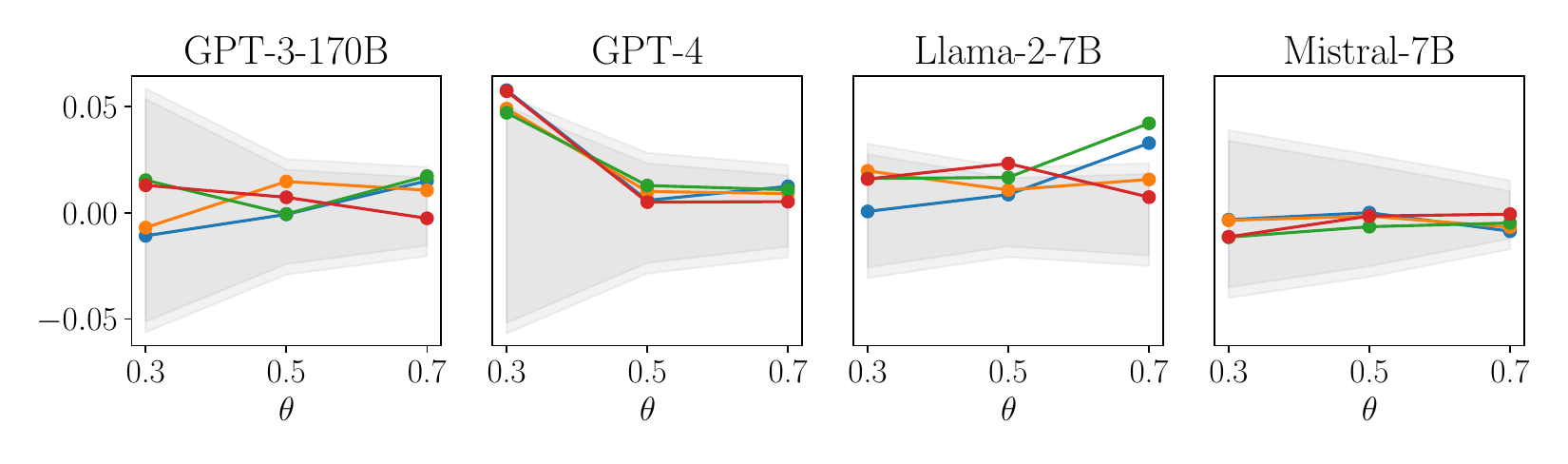}}
\subfigure[$T_{1,g}$ for $g(z)=z, n=100, m=200$]{\includegraphics[width=0.48\linewidth]{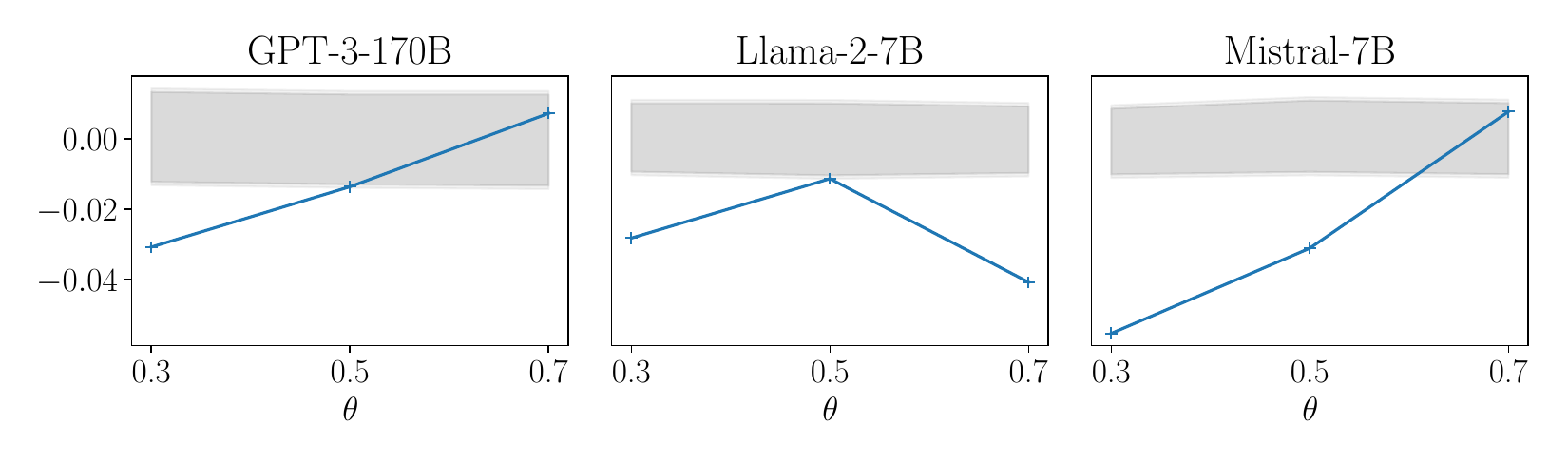}}
\subfigure[$T_{2,k}$ for $k\in\{2,3,4,5\}, n=100, m=200$]{\includegraphics[width=0.48\linewidth]{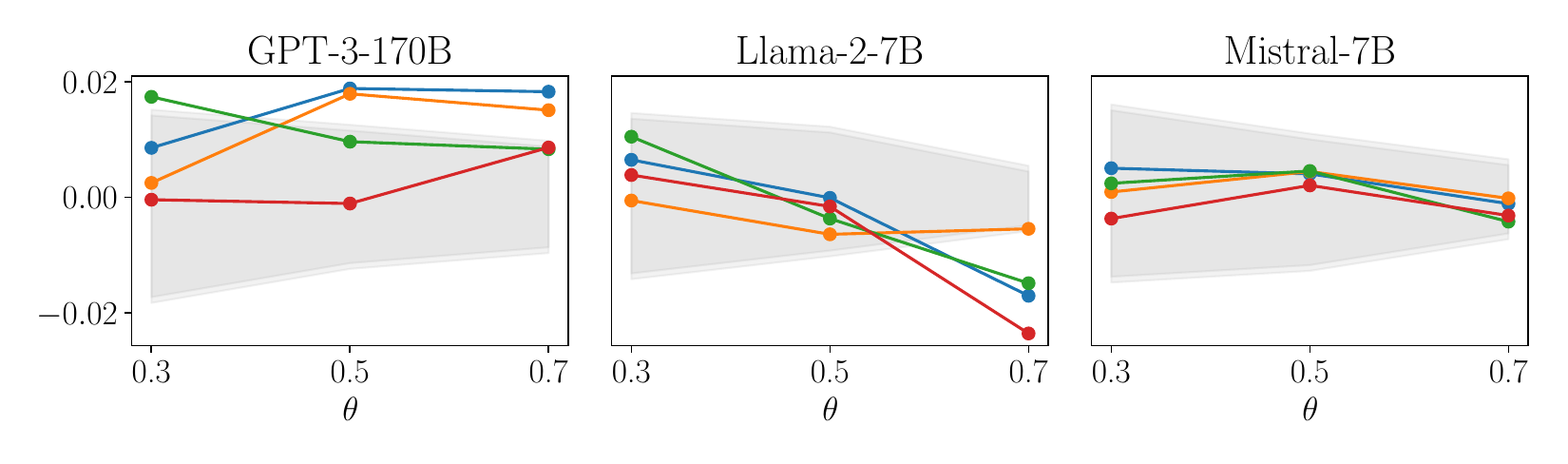}}
\caption{Checking the martingale property: results for the Bernoulli experiments for all choices of $(n,m)$ in the setting of Fig.~\ref{fig:cid-bern-main}. Note that we drop \texttt{gpt-4} for $n=100$ due to API limitations (as discussed in App.~\ref{app:additional experiment details}). 
}\label{fig:cid-bern-appendix}
\end{figure*}

\begin{figure*}[p]
\centering 
\subfigure[$T_{1,g}$ for $g(z)=z, n=20, m=10$]{\includegraphics[width=0.48\linewidth]{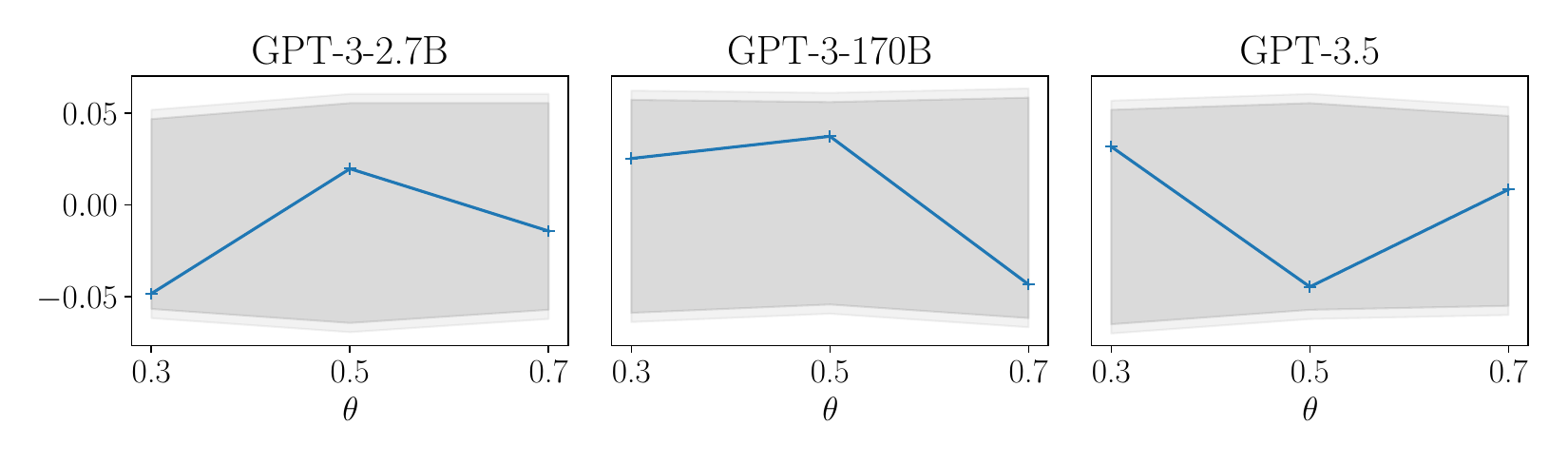}}
\subfigure[$T_{2,k}$ for $k\in\{2,3,4,5\}, n=20, m=10$]{\includegraphics[width=0.48\linewidth]{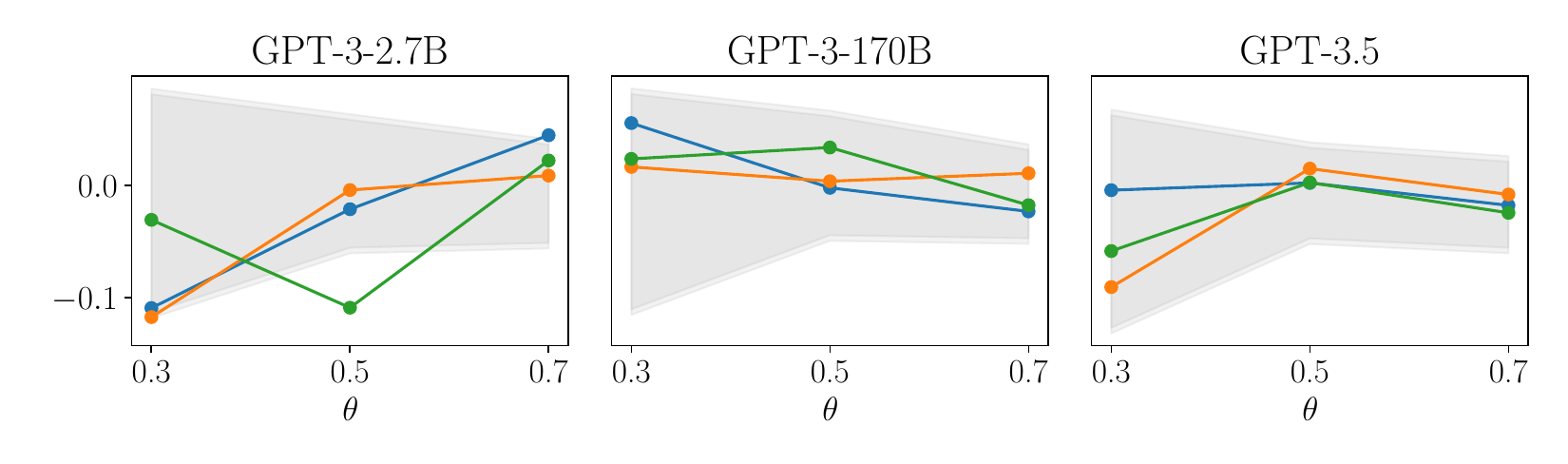}}
\subfigure[$T_{1,g}$ for $g(z)=z, n=20, m=10$]{\includegraphics[width=0.48\linewidth]{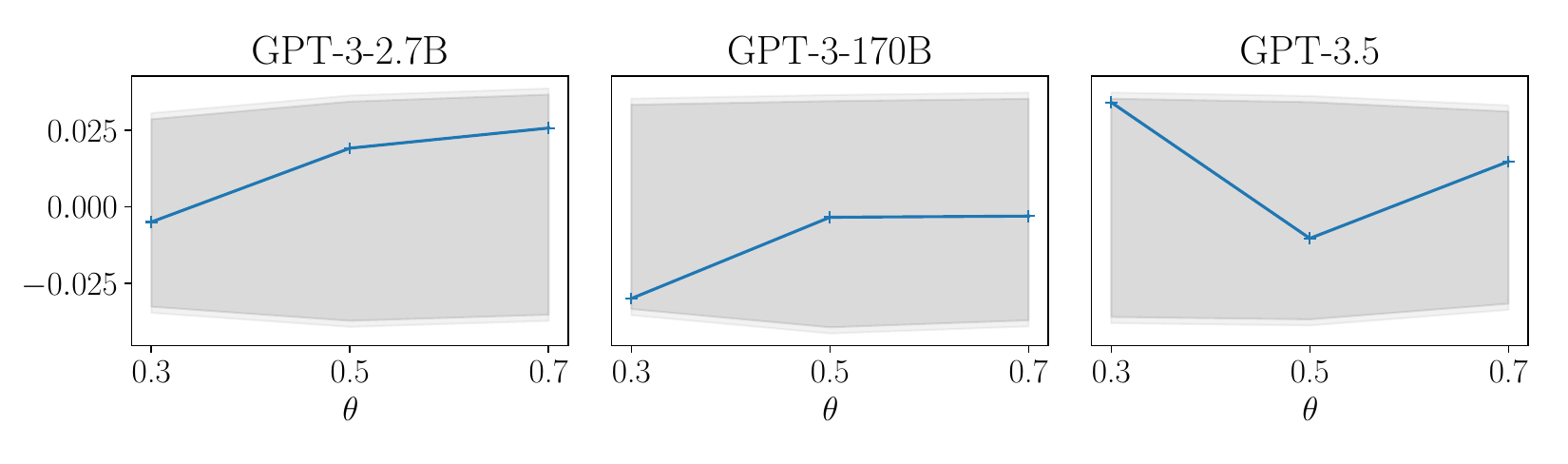}}
\subfigure[$T_{2,k}$ for $k\in\{2,3,4,5\}, n=20, m=10$]{\includegraphics[width=0.48\linewidth]{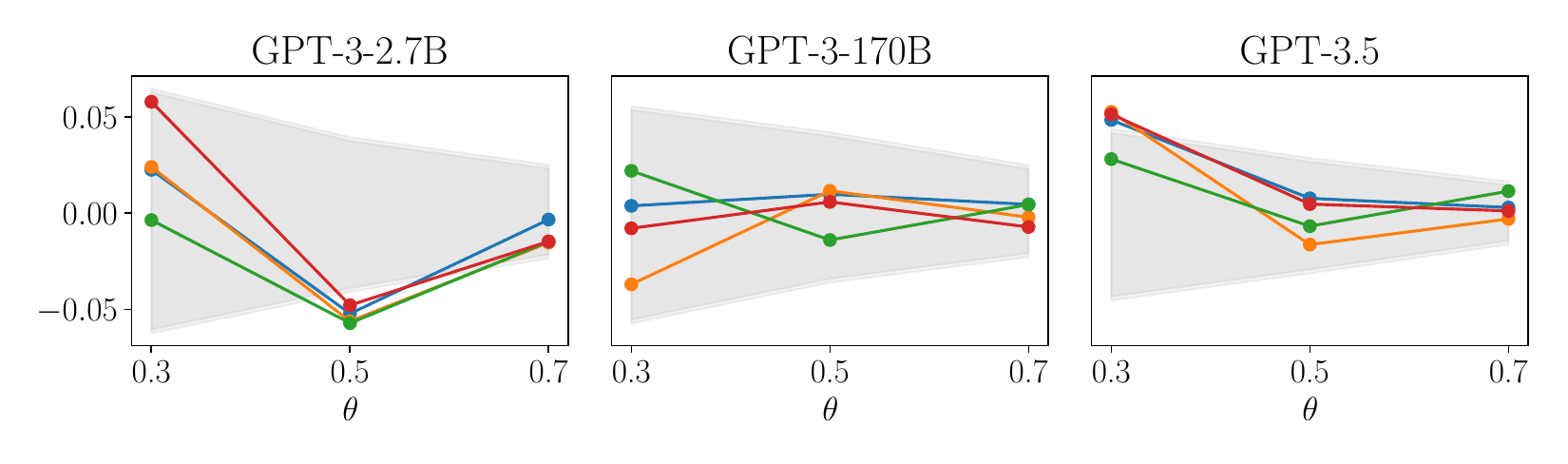}}
\subfigure[$T_{1,g}$ for $g(z)=z, n=100, m=50$]{\includegraphics[width=0.48\linewidth]{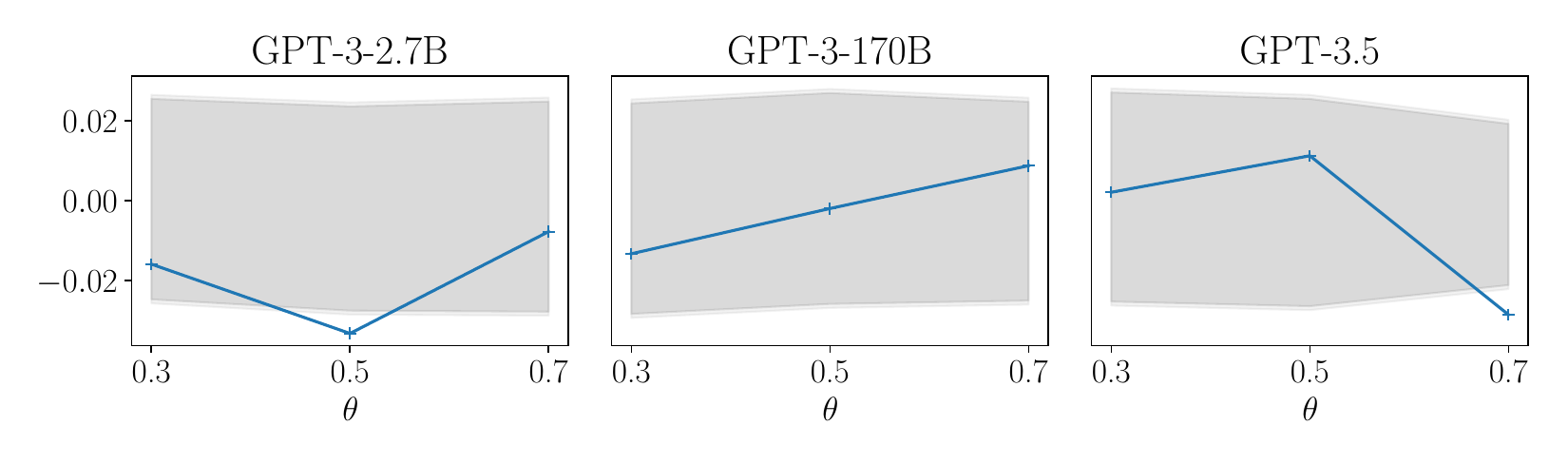}}
\subfigure[$T_{2,k}$ for $k\in\{2,3,4,5\}, n=100, m=50$]{\includegraphics[width=0.48\linewidth]{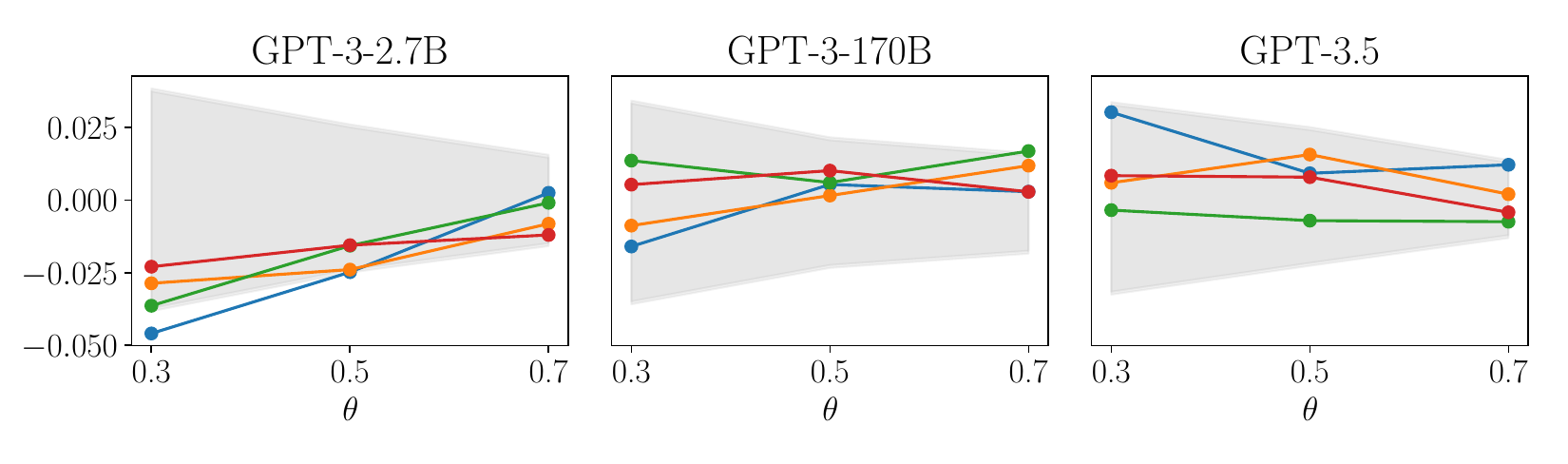}}
\subfigure[$T_{1,g}$ for $g(z)=z, n=20, m=40$]{\includegraphics[width=0.48\linewidth]{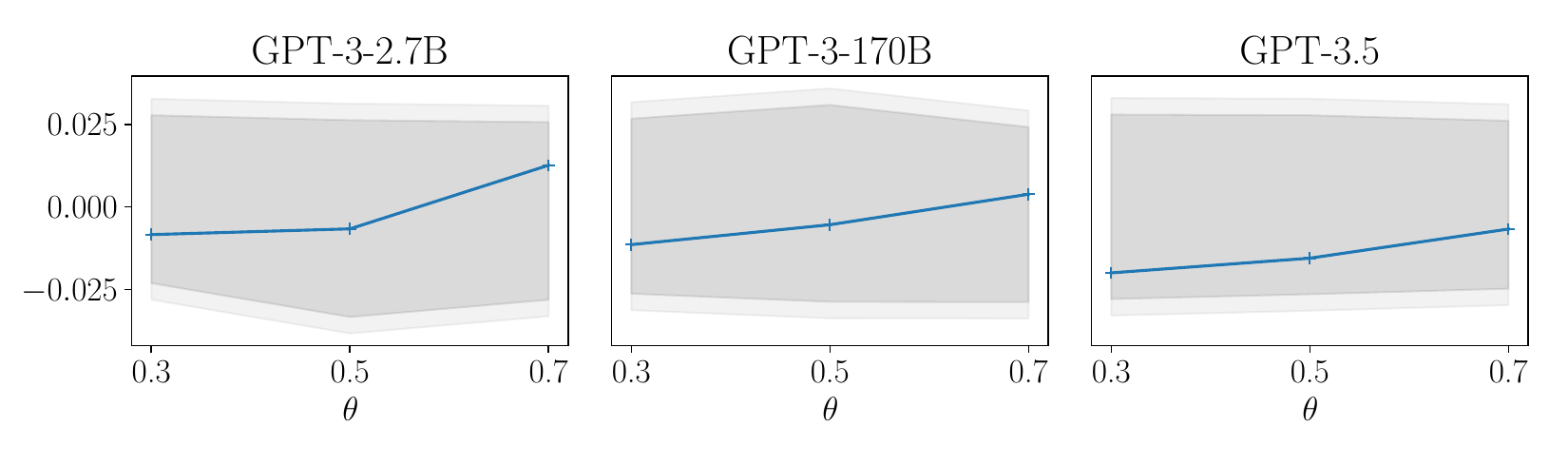}}
\subfigure[$T_{2,k}$ for $k\in\{2,3,4,5\}, n=20, m=40$]{\includegraphics[width=0.48\linewidth]{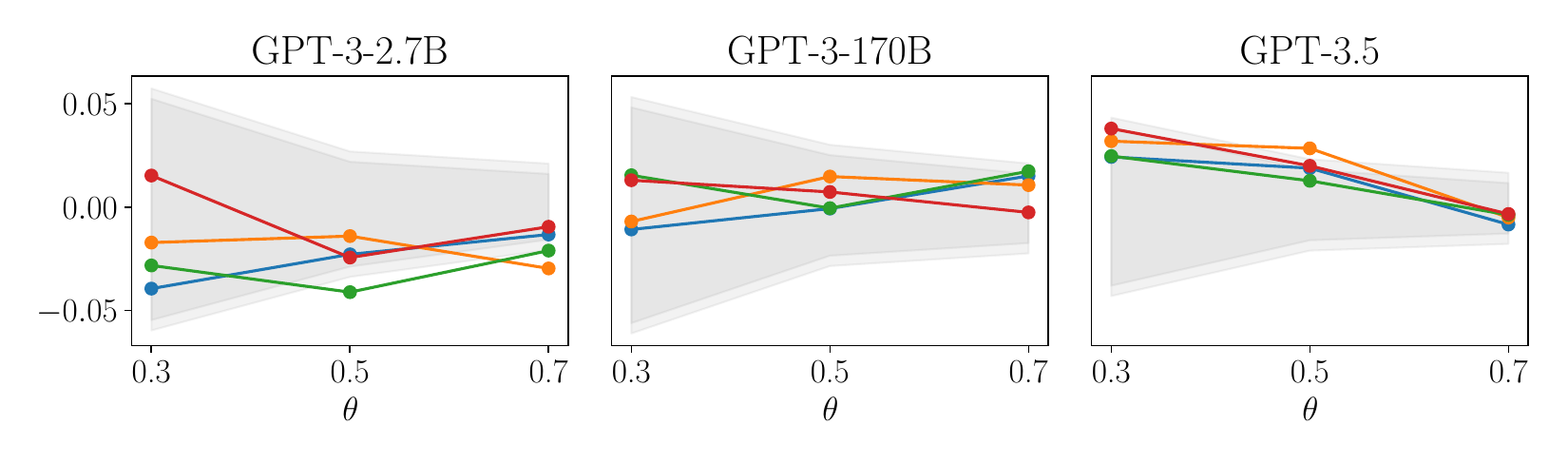}}
\subfigure[$T_{1,g}$ for $g(z)=z, n=50, m=100$]{\includegraphics[width=0.48\linewidth]{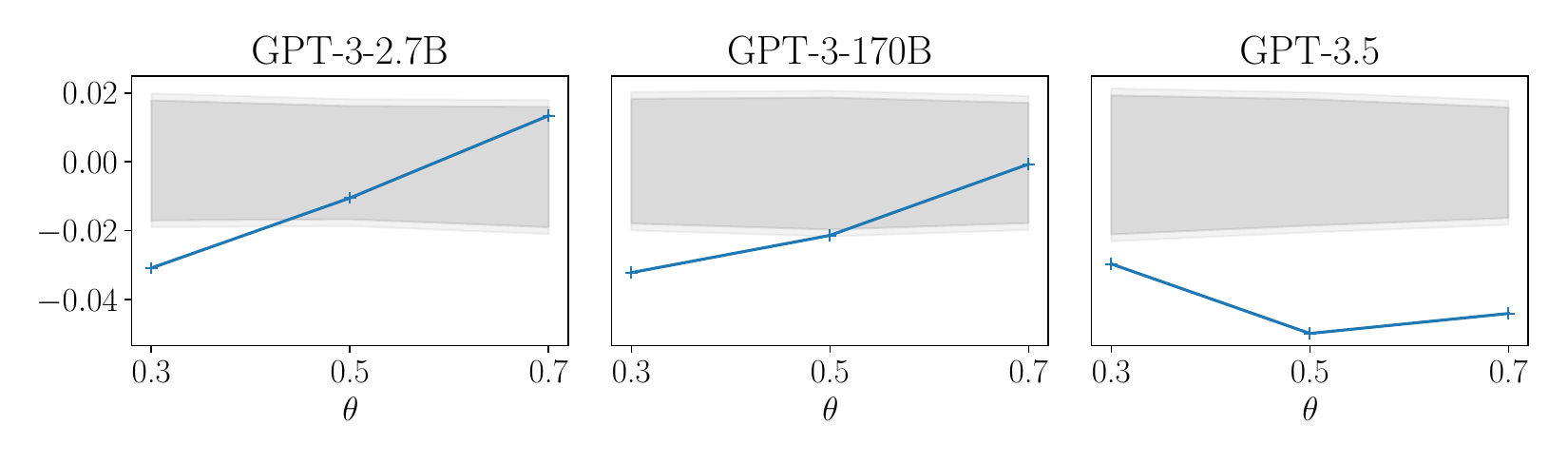}}
\subfigure[$T_{2,k}$ for $k\in\{2,3,4,5\}, n=50, m=100$]{\includegraphics[width=0.48\linewidth]{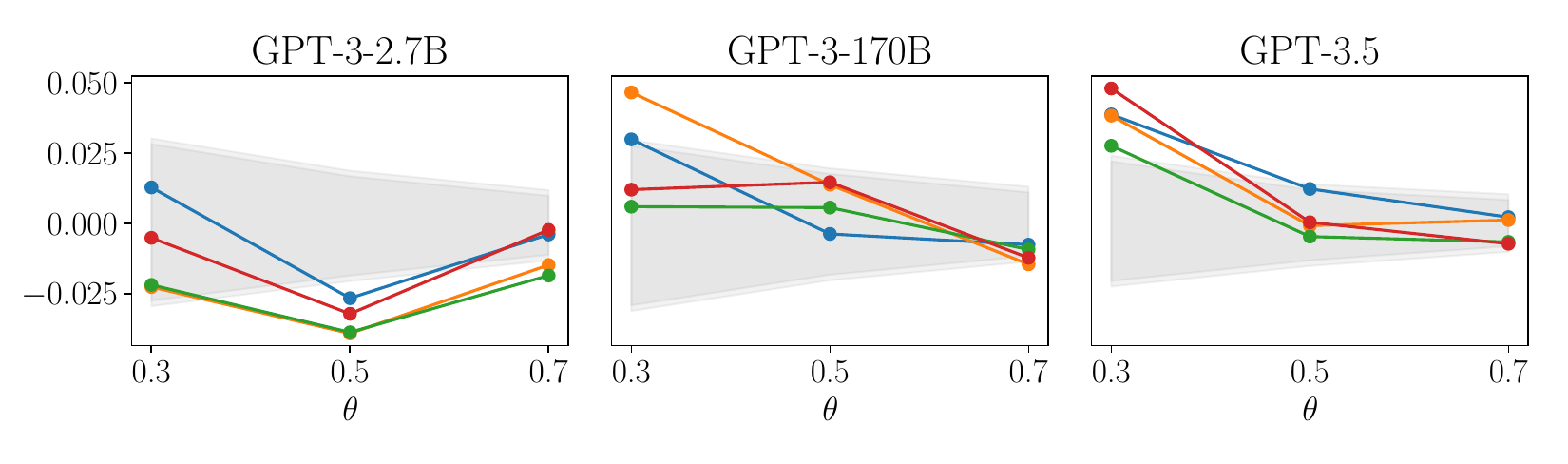}}
\subfigure[$T_{1,g}$ for $g(z)=z, n=100, m=200$]{\includegraphics[width=0.48\linewidth]{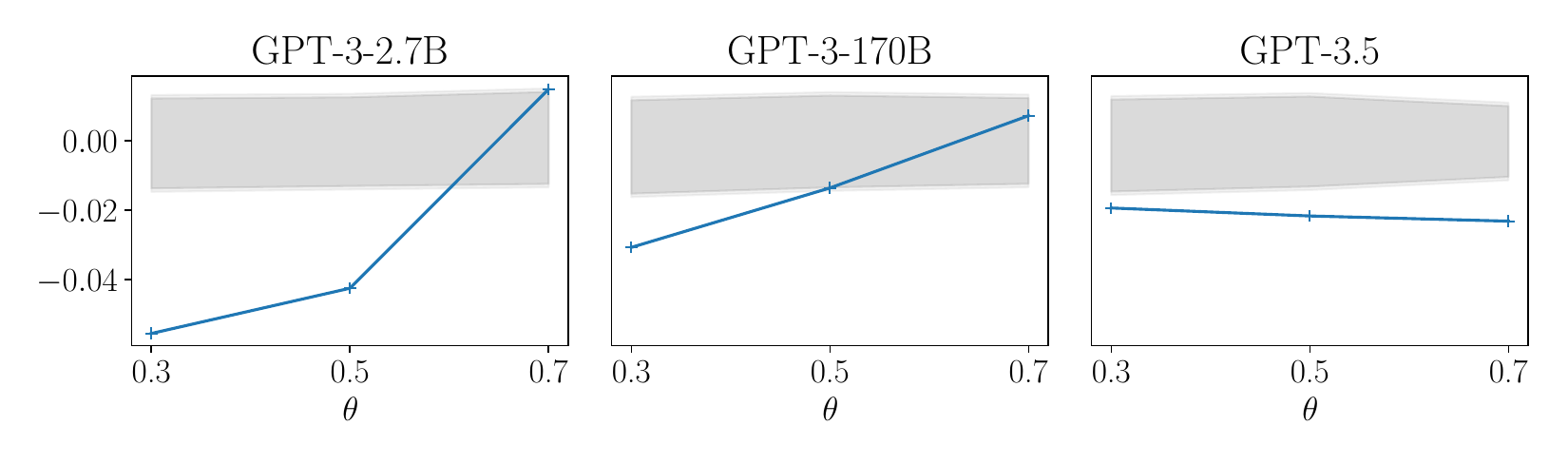}}
\subfigure[$T_{2,k}$ for $k\in\{2,3,4,5\}, n=100, m=200$]{\includegraphics[width=0.48\linewidth]{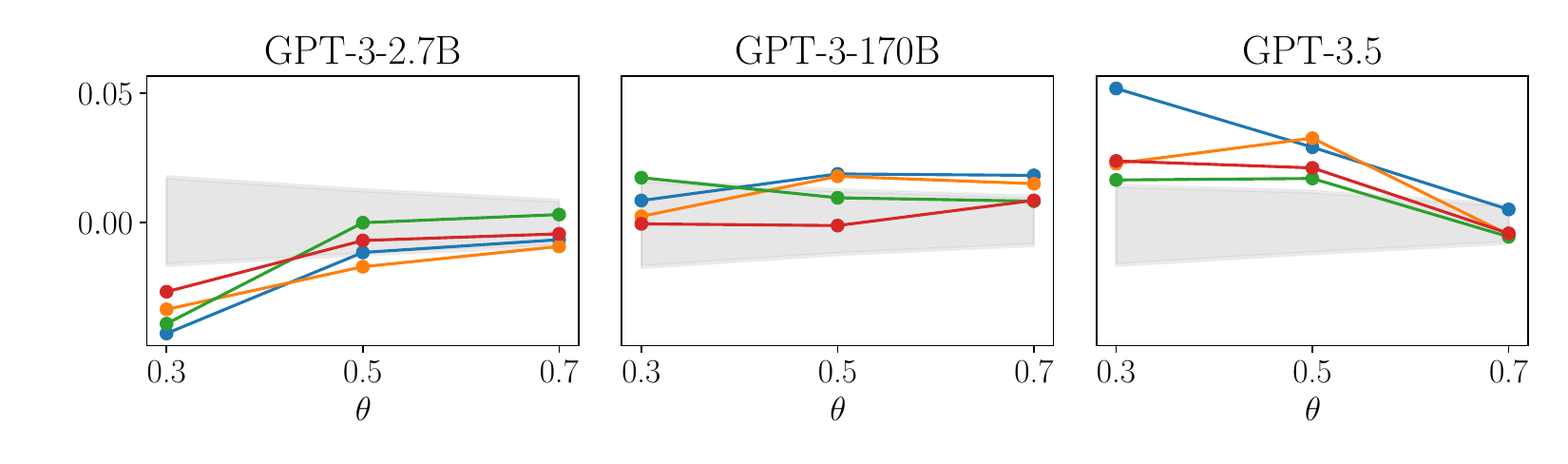}}
\caption{Checking the martingale property: results for \texttt{gpt-3-2.7b} and \texttt{gpt-3.5} %
in the setting of Fig.~\ref{fig:cid-bern-main}. %
}\label{fig:cid-bern-appendix-gpt}
\end{figure*}

\begin{figure*}[p]
\centering 
\subfigure[$n=20,m=10$]{\includegraphics[width=0.48\linewidth]{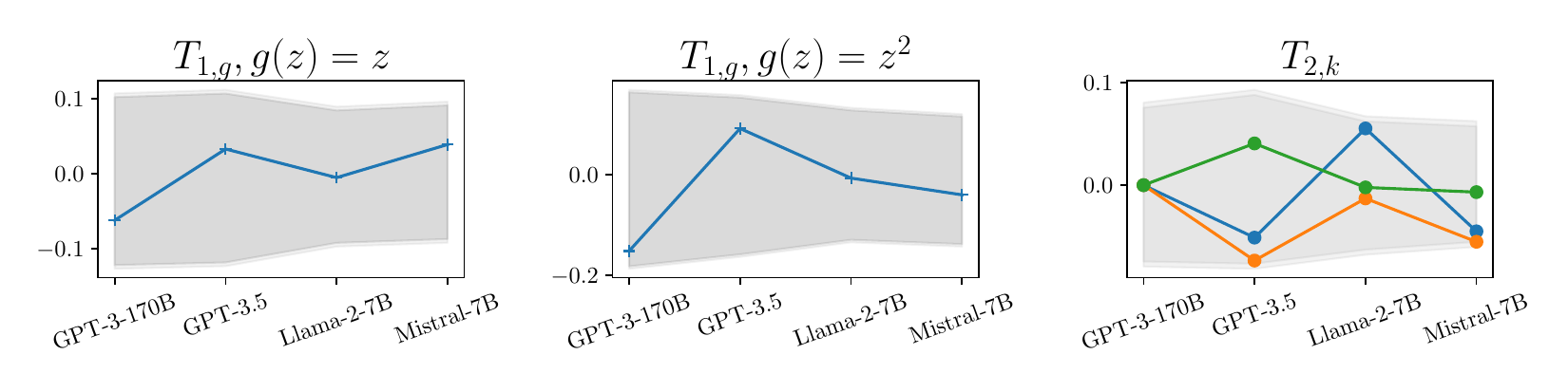}}
\subfigure[$n=20,m=40$]{\includegraphics[width=0.48\linewidth]{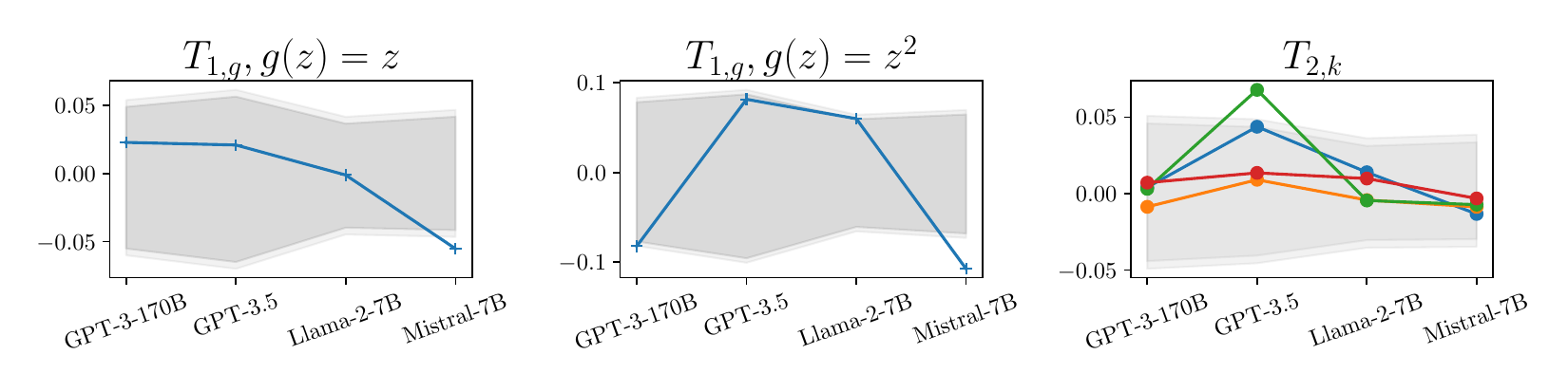}}
\subfigure[$n=50,m=25$]{\includegraphics[width=0.48\linewidth]{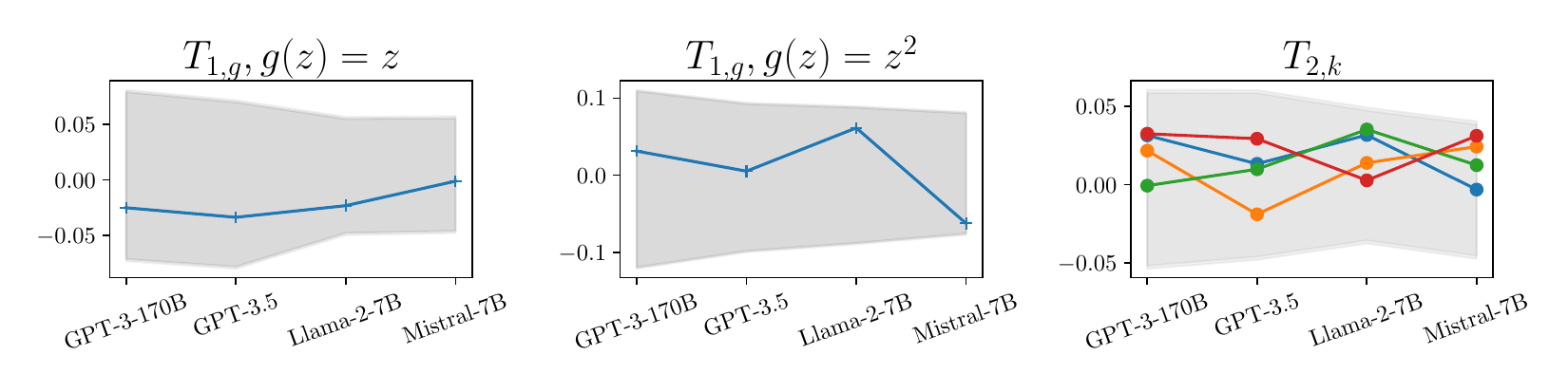}}
\subfigure[$n=50,m=100$]{\includegraphics[width=0.48\linewidth]{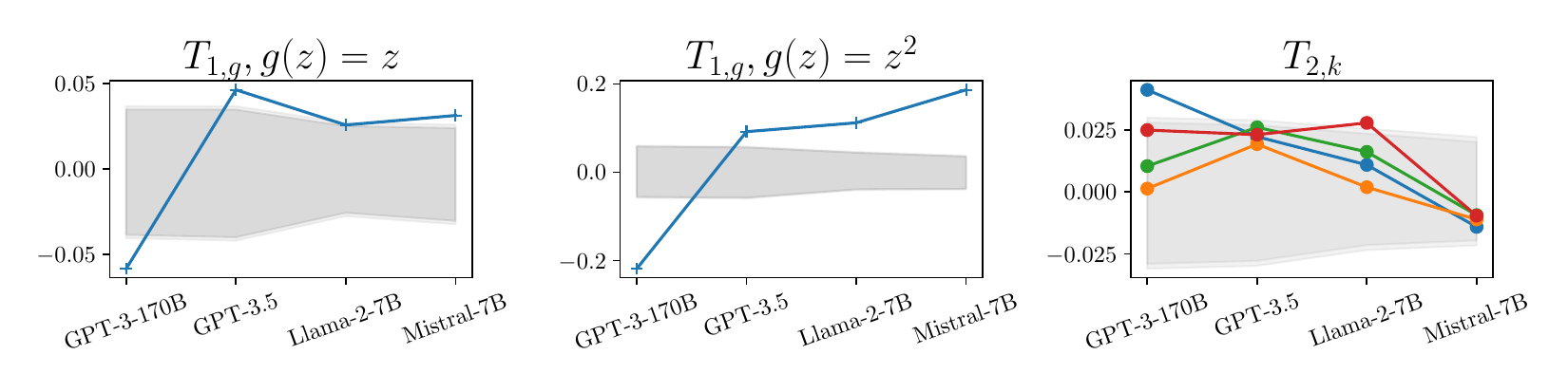}}
\caption{Checking the martingale property: results for the Gaussian experiments with $\theta=0$. See Fig.~\ref{fig:cid-gauss-main} for details.}\label{fig:cid-gauss-appendix-m0}
\end{figure*}

\begin{figure*}[t]
\centering 
\subfigure[$n=20,m=10$]{\includegraphics[width=0.48\linewidth]{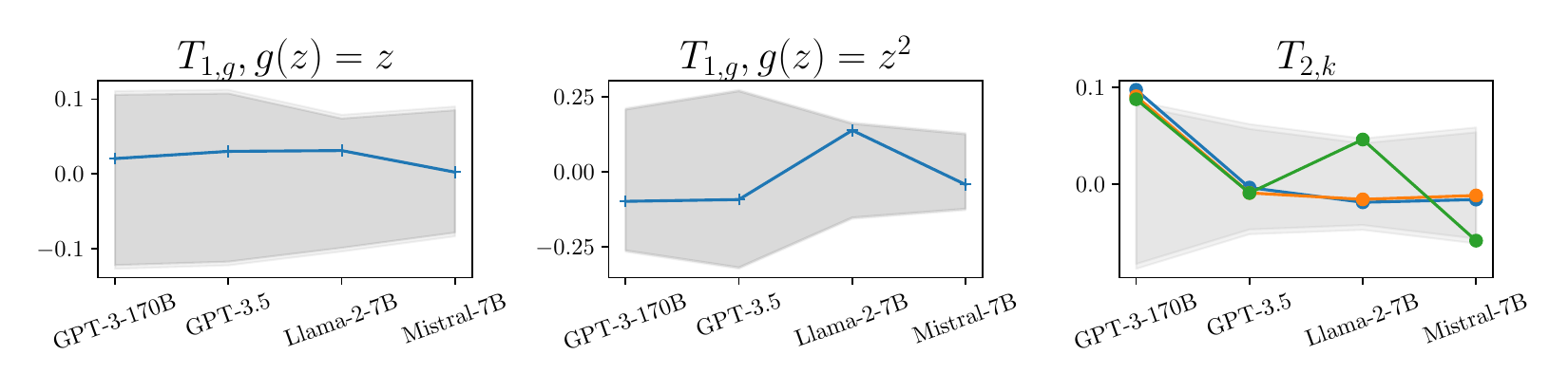}}
\subfigure[$n=20,m=40$]{\includegraphics[width=0.48\linewidth]{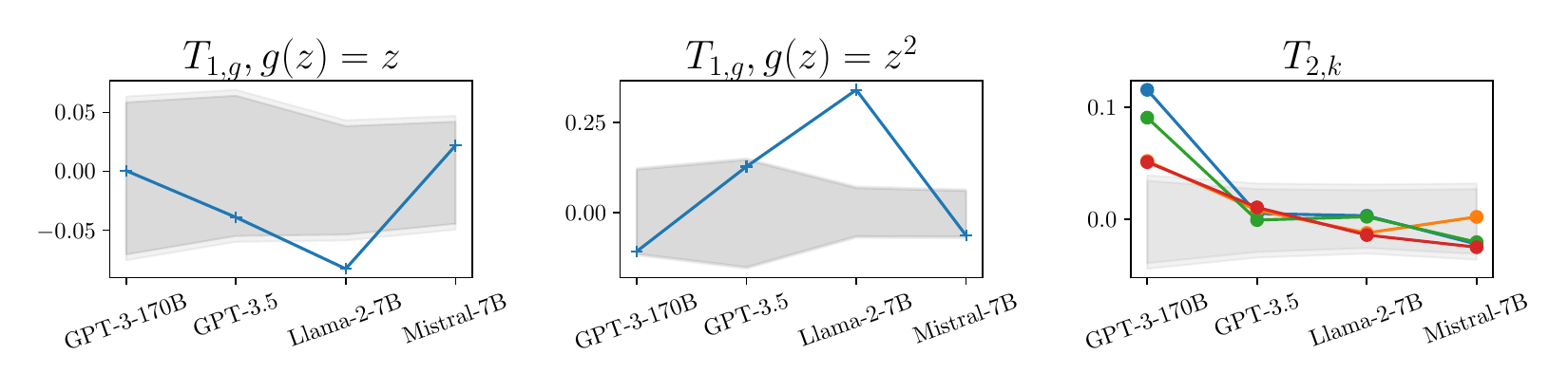}}
\subfigure[$n=50,m=25$]{\includegraphics[width=0.48\linewidth]{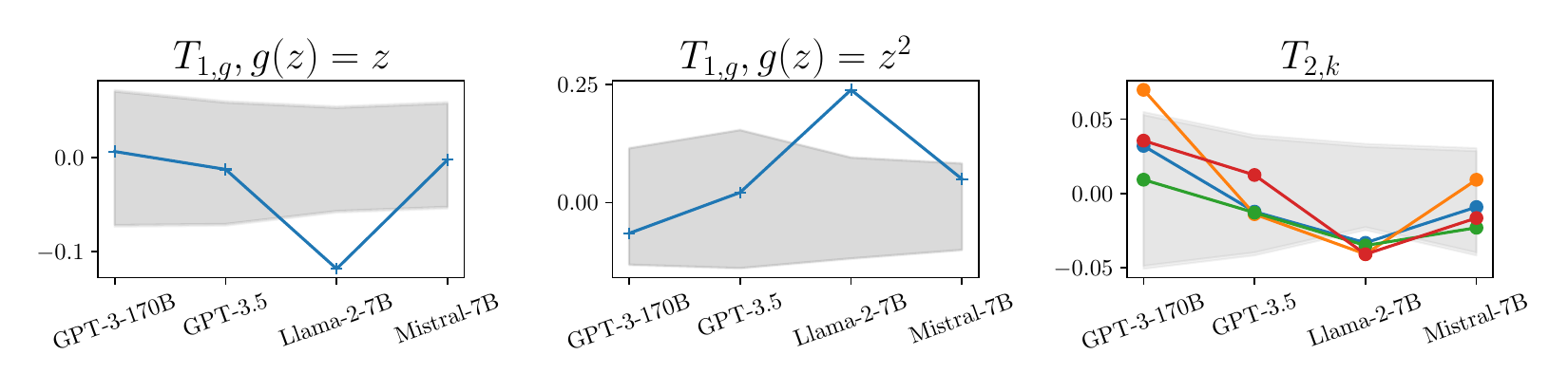}}
\subfigure[$n=50,m=100$]{\includegraphics[width=0.48\linewidth]{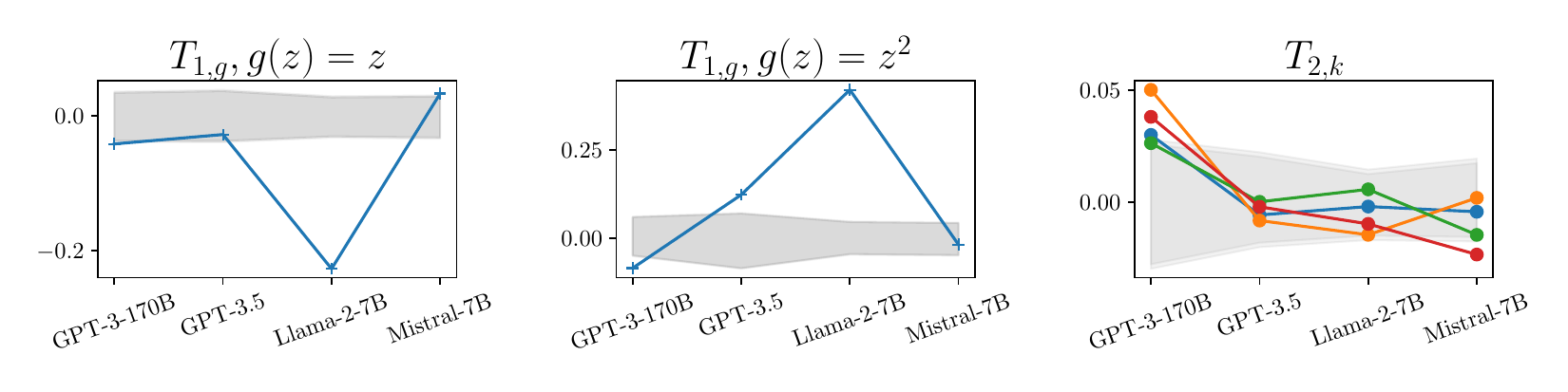}}
\subfigure[$n=100,m=50$]{\includegraphics[width=0.48\linewidth]{Figures/agg/gauss_all_-1_100_50_3d_35.pdf}}
\subfigure[$n=100,m=200$]{\includegraphics[width=0.48\linewidth]{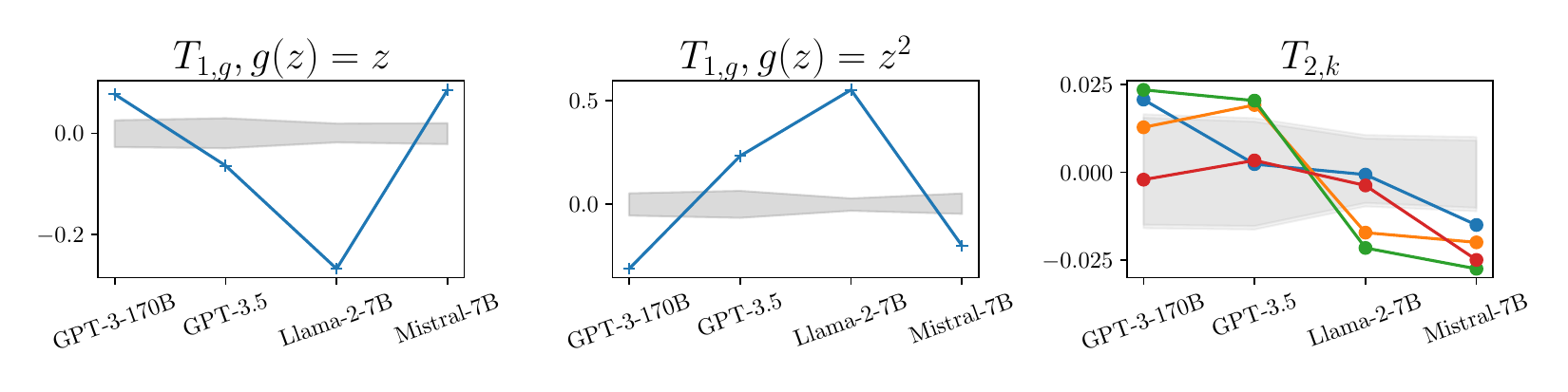}}
\caption{Checking the martingale property: results for the Gaussian experiments with $\theta=-1$. See Fig.~\ref{fig:cid-gauss-main} for details.}\label{fig:cid-gauss-appendix-m1}
\end{figure*}

\paragraph{Additional details for the natural language experiment.}
For the natural language experiment, we modify the %
scheme as follows: 
we split the ICL dataset and the imputations into two sequences $
(\{Y_{i_{0,k}}\}_{k=1}^{n_1+m_1},  
\{Y_{i_{1,k}}\}_{k=1}^{n_0+m_0})
$
based on the value of $X_i$. Subsequently,  %
either sequence contains i.i.d.~Bernoulli random variables with a different mean, and any Bayesian ICL model with a correctly specified likelihood must produce imputations following a separate Bayesian posterior for Bernoulli data. 
Thus, we can apply our Bernoulli diagnostics separately to both sequence. 
This modification allows us to focus on LLMs' conditional %
predictive distributions of the form  $p_M(Y_{i+1}\vert X_{i+1}, Z_{1:i})$, which is more relevant in practice. 

\subsection{Further Experimental Results and Discussion}\label{app:further-experimental-results}

\paragraph{Full results: Bernoulli experiments.}
Figs.~\ref{fig:cid-bern-appendix} and \ref{fig:cid-bern-appendix-gpt} report the full results for the Bernoulli experiment in the setting of Fig.~\ref{fig:cid-bern-main} ($m\in\{n/2,2n\})$, where we also visualise the $o(1/n)$ `acceptable deviation' (\S\ref{app:acceptable-dev}) using a light shade with width $0.1/n$. 
Consistent with the results in Fig.~\ref{fig:cid-bern-main}, 
for all models except the least capable \texttt{gpt-3-2.7b}, 
the martingale property is generally satisfied in the short-horizon scheme ($m=n/2$), but increasingly violated as we move to $m=2n$. 

The results for \texttt{gpt-3-2.7b} provide a sanity check of our experiment setup: Its unsatisfactory performance shows that our tasks require nontrivial ICL capabilities which are known to be absent in \texttt{gpt-3-2.7b} \citep{brown2020language}. As another sanity check 
we provide the results for $m=10n$ in Fig.~\ref{fig:bern-hf-long}, where we drop all GPT models due to limitations with its API. As we can see, in this setting where the sampling horizon becomes even longer, deviation from the martingale property also becomes more severe. The consistently large negative value of $T_{1,g}$ indicates a continual upward bias towards $1$, which demonstrates the `creation of new knowledge' phenomenon discussed in \S\ref{sec:martingale}.

\paragraph{Full results: Gaussian experiments.}
We report additional results for the Gaussian experiment in Fig.~\ref{fig:cid-gauss-appendix-m0} ($\theta=0$) and Fig.~\ref{fig:cid-gauss-appendix-m1} ($\theta=-1$). As we can see, all models generally demonstrate a deviation from the martingale property when $\theta=-1$, but with $\theta=0$ they may often appear to satisfy the property within a shorter horizon ($m=n/2$). 
Results for $\theta=1$ are similar to the $\theta=-1$ case and thus omitted. 
We note that in many cases the predictive distribution cannot be matched to any Bayesian posterior with the correct likelihood: for the latter the sample variance should be greater than $1$, the likelihood variance, but this is often not true for the LLMs. For example, for \texttt{gpt-3.5} in the setting of Fig.~\ref{fig:cid-gauss-main} we find the sample variance to be $0.711<1$ (95\% CI: $[0.680, 0.742]$).

\begin{figure}[htb]
    \centering
    \subfigure[$T_{1,g},g(z)=z$]{
    \includegraphics[width=0.65\linewidth]{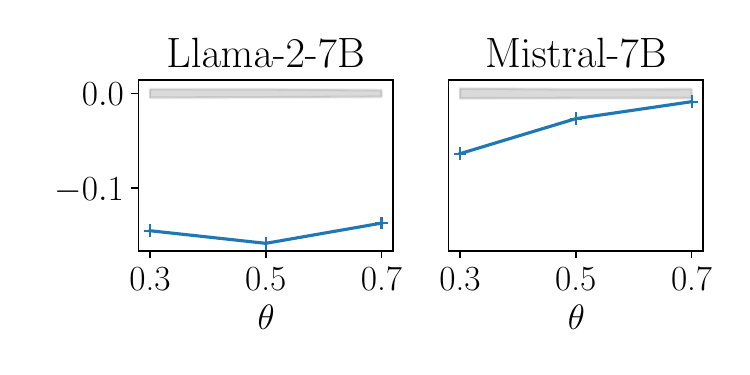}}
    \subfigure[$T_{2,k},k\in\{2,3,4,5\}$]{
    \includegraphics[width=0.65\linewidth]{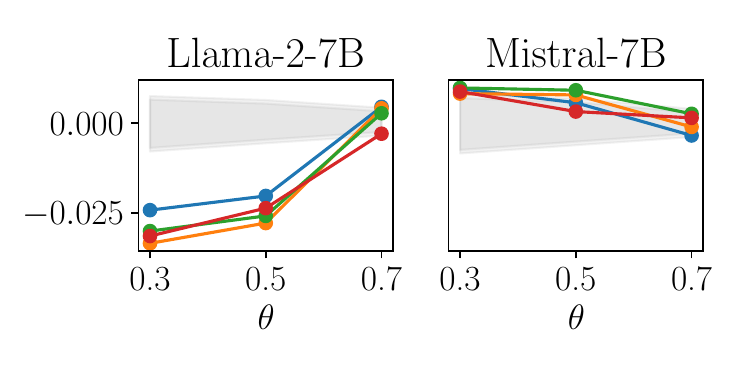}}
    \caption{
    Checking the martingale property on Bernoulli experiments: additional result with $n=100, m=10n$. See Fig.~\ref{fig:cid-bern-main} for details.
    }
    \label{fig:bern-hf-long}
\end{figure}

\begin{figure}[htb]
    \centering
    \includegraphics[width=0.98\linewidth]{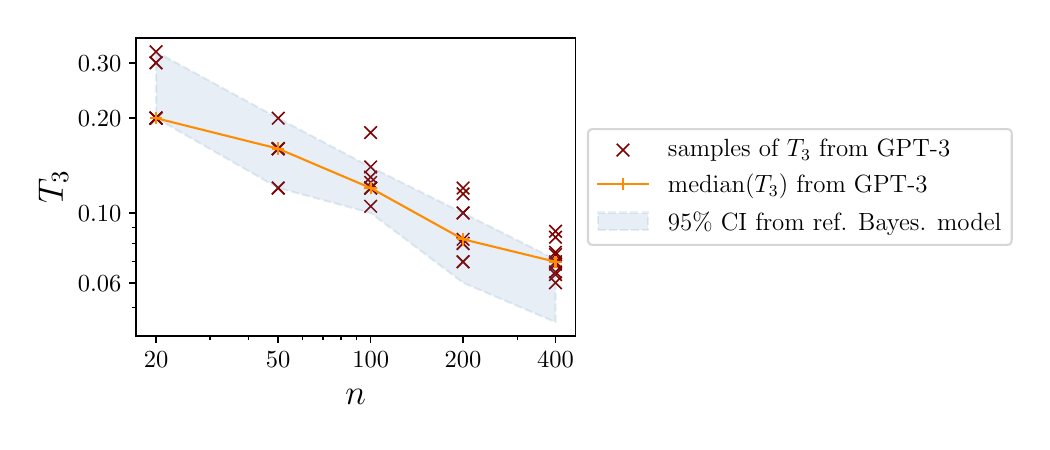}
    \caption{Scaling of epistemic uncertainty: samples of the test statistic $T_3$ evaluated on \texttt{gpt-3-170b}, compared with the 95\% CI from the reference Bayesian model. 
    }
    \label{fig:scaling-eps-gpt3-emp}
\end{figure}

\paragraph{Additional results: Scaling of epistemic uncertainty.} 
To avoid clutter, in 
Fig.~\ref{fig:scaling-epis} we have plotted the sample median of the test statistic $T_3$ from various models, and in that aspect \texttt{gpt-3-170b} appears to be close to the reference Bayesian model when $n$ is smaller. 
Here we note that a deviation becomes more evident  
if we compare the individual samples of $T_3$  (obtained from independent runs) against bootstrap CIs from the reference Bayesian model, as shown in Fig.~\ref{fig:scaling-eps-gpt3-emp}.  
\begin{figure}[htb]
    \centering
    \subfigure[\texttt{gpt-3-2.7b}, NLP, pretrained]{\includegraphics[width=0.92\linewidth]{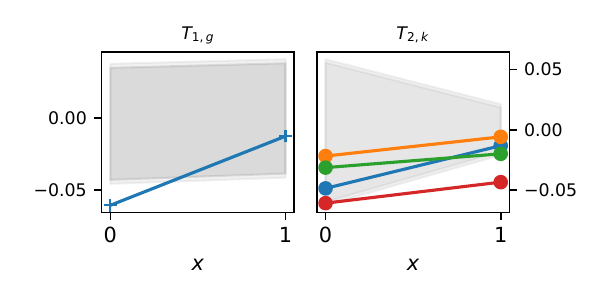}} 
    \subfigure[\texttt{gpt-3-2.7b}, NLP, finetuned]{\includegraphics[width=0.92\linewidth]{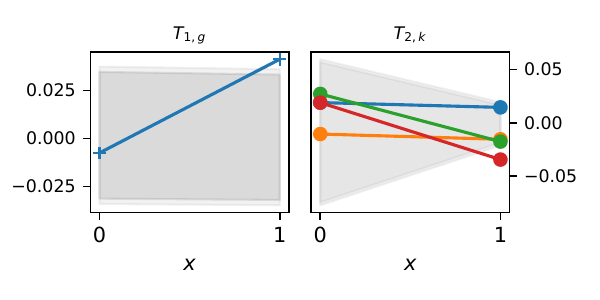}}
    \subfigure[\texttt{gpt-3-2.7b}, Bernoulli ($n=50,m=100$), pretrained]{
    \includegraphics[width=0.92\linewidth]{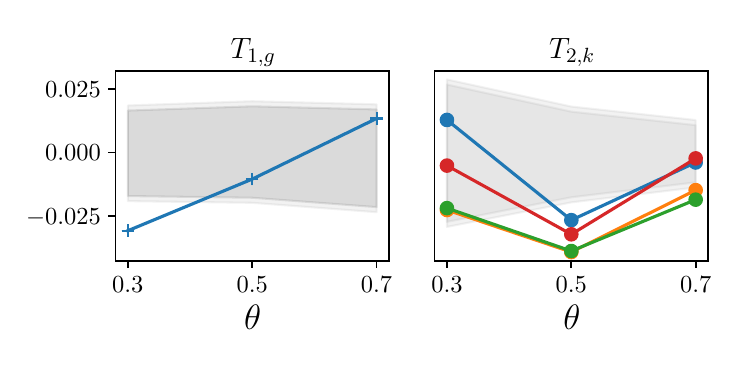}}
    \subfigure[\texttt{gpt-3-2.7b}, Bernoulli ($n=50,m=100$), finetuned]{
    \includegraphics[width=0.92\linewidth]{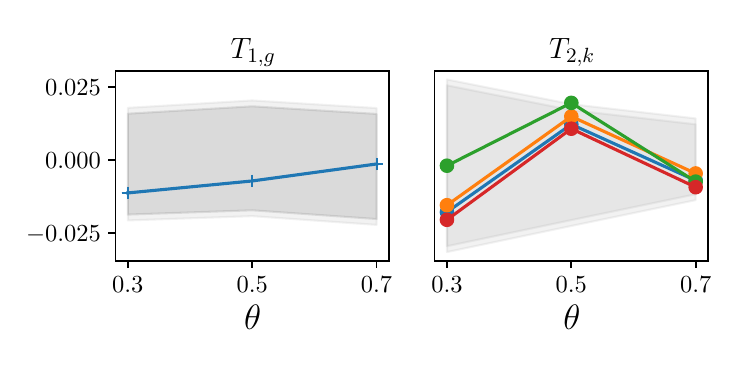}}
    \caption{Checking the martingale property: comparison of \texttt{gpt-3} models 
    before and after fine-tuning on the NLP (Fig.~\ref{fig:icl-both}) and Bernoulli (Fig.~\ref{fig:cid-bern-main}) datasets.} 
    \label{fig:fine-tune}
\end{figure}

\paragraph{Additional results with fine-tuned models.} 
Some previous works \citep{zhang2023and,jiang2023latent} studied ICL under the assumption that the LLM has been perfectly pretrained on the ICL test distributions. 
While such assumptions are somewhat unrealistic, it may still be interesting to investigate whether finetuning on datasets that are similar to the ICL test distribution could lead to a closer-to-Bayesian behaviour for ICL. 
To this end we finetune \texttt{gpt-3-2.7b} models on Bernoulli and synthetic NLP datasets with randomly sampled parameters, and repeat the checks in \S\ref{sec:CID-check} on the finetuned models.\footnote{
We use OpenAI's finetuning service which determines optimisation hyperparameters by validation loss. 
For the Bernoulli dataset, we sampled $10^4$ sequence for training, each with an expected length of $75$; the true parameter $\theta$ is sampled from the uniform distribution on $[0,1]$. 
For the NLP dataset, we sampled $5000$ sequences for training, each with an expected length of $85$; the two Bernoulli parameters that determine the prompt distribution are sampled from a $\mathrm{Beta}(0.5, 0.5)$ distribution.}
The results are visualised in Fig.~\ref{fig:fine-tune}. We can see that 
the finetuned models may indeed demonstrate a better adherence to the martingale property, but they do not always pass the checks. 

\paragraph{Comparison of LLMs with and without instruction tuning.} 
Among all LLMs evaluated, \texttt{gpt-3.5} and \texttt{gpt-4} generally demonstrate the worst performance in our evaluations, and incidentally they are the only LLMs that have undergone instruction tuning. 
The comparison between \texttt{gpt-3-170b} and \texttt{gpt-3.5} (see Fig.~\ref{fig:scaling-epis} and Fig.~\ref{fig:cid-bern-appendix-gpt}) is particularly interesting since the two models are generally similar, with a main difference being the presence of instruction tuning. 
These observations seem to suggest that instruction tuning may have exacerbated the %
non-Bayesian ICL behaviour. %
Such an explanation would be broadly consistent with the previous findings 
that instruction tuning generally causes the LLM to produce less calibrated uncertainty estimates \citep{openai2023gpt4,gruver2023large,kalai2023calibrated}.

\paragraph{Could the LLMs correspond to `unreasonable Bayesian models'?} 
As discussed in the main text, our findings suggest that the behaviours of \texttt{gpt-3.5} and \texttt{llama-2-7b} are highly unlikely to correspond to any `reasonable' Bayesian models in the Bernstein von-Mises sense. 
Generally speaking, to conduct any statistical test with a reasonable level of power it is necessary to impose some regularity restrictions on the null hypothesis to be tested. 
Moreover, 
it could be similarly concerning 
if the LLMs correspond to any `unreasonable' Bayesian model that does not satisfy the regularity conditions in the Bernstein von-Mises theorem. 
Nonetheless, in the setting above we can also provide some informal discussion on why the LLMs are unlikely 
to be `unreasonable' Bayesian models (e.g., one with an approximately degenerate prior), by comparing the results across different choices of $n$. 
Specifically, 
for \texttt{gpt-3.5}, 
its small-sample behaviour in Fig.~\ref{fig:scaling-epis} can only be explained as a Bayesian model with a very strong prior that has the bulk of its mass near the true parameter; yet this would contradict its larger-than-regular posterior spread when $n$ is large. 
For \texttt{llama-2-7b}, its large-sample behaviour could only be explained with the exact opposite (e.g., a $\mathrm{Beta}(100, 100)$ prior); yet that should have led to a much larger IQR when $n$ is small.

\section{Related Work}
\label{sec:Related work}

\paragraph{In-context learning as Bayesian inference. }
Numerous papers have explained ICL as performing some form of Bayesian Inference. 
The hypothesis is likewise studied in %
\citet{jiang2023latent,wang2023large} and the concurrent work of \citet{ye2024pretraining}. 
It is also covered by 
\citet{zhang2023and} if we restrict to exchangeable demonstrations. %
Closely related are the works of \citet{akyurek2022learning,panwar2024incontext} which demonstrate that high-capacity transformers pretrained with square loss may recover the Bayes predictor. %

\citet{xie2021explanation} studied ICL in a setting where the the LLM is perfectly trained on a pretraining distribution defined by a 
Hidden Markov Model (HMM). 
Under this and further assumptions, they prove that the LLM must implicitly perform Bayesian inference to infer
a latent concept of the prompt. %
Strictly speaking, their assumptions do not exactly match our hypothesis, because their Bayesian model %
employs a likelihood that is misspecified for ICL: it does not assume $\{Z_i\}$ is conditionally i.i.d.~or exchangeable. 
However, their additional assumptions render the ICL behaviour similar to that of another Bayesian model that assumes conditionally i.i.d.~observations: 
when considering Eq.~(8-10) in their work, which imply that the log likelihood of their Bayesian model is well approximated by a Bayesian model assuming conditional i.i.d.~observations. %
In this regard their analysis is connected to our hypothesis, as it applies to the Bayesian model we study. 
It is important to note that their assumptions 
have been crucial in their proof for sample efficiency. 
More broadly, for any ICL predictor to be sample efficient on exchangeable $\{Z_i\}$, it is perhaps reasonable to expect the predictor to (approximately) recognise the exchangeable nature of $\{Z_i\}$, where our hypothesis would apply.

We review the other works in brief in the following.
\citet{bayes_blog} dicussed high-level connections between \citet{xie2021explanation} and various notions of exchangeability.
\citet{hahn2023theory}[Section 1.4] relates to \citep{xie2021explanation} as it can similarly be understood in terms of Bayesian inference, with the difference that they view the training tasks to be open-ended and compositional, in contrast to the finite nature of an HMM.
\citet{wang2023large} likewise takes a Bayesian viewpoint, which they utilise to select the ICL dataset optimally.
\citet{jiang2023latent} explains various phenomena of the `emergent abilities' of LLMs, such as in-context learning and chain-of-thought prompting, through Bayesian inference on the common distribution underlying natural languages.
\citet{zhang2023and} show that ICL implicitly uses a Bayesian model averaging.
\citet{griffiths2006optimal} recover the prior distributions in LLMs for everyday observations, such as the time of movies.

\paragraph{Theories for in-context learning. }
Numerous theoretical models and frameworks beyond Bayesian inference exist which aim at understanding and formalising ICL.
We refer to \cite{dong2022survey} for a detailed survey on in-context learning.
\citet{akyurek2022learning} prove that transformer-based architectures can implement classical learning algorithms such as linear models and ridge regression.
\citet{bai2023transformers} extend this work by demonstrating that ICL via transformers can implement and even braoder set of algorithms, including convex risk minimisation algorithms and gradient descent, where the model intrinsically selects a different learning algorithm based on the task at hand.
\citet{singh2023transient} shows that the ability of performing ICL algorithms such as Bayesian inference may be a transient phenomenon which produces highest accuracy during certain stages of pretraining an LLM.
\citet{raventos2023pretraining} show that the ability of in-context learning to tasks unseen during training by picking the right learning algorithm depends on the task diversity during training.

\paragraph{Input order dependence of Large language models. }
Previous work has found a dependence of LLMs on the order in which an input sequence is presented. 
\citet{lu2021fantastically} demonstrate that input order can significantly change the performance of an LLM in text classification tasks from ``state-of-the-art'' to ``random guess''.
In the context of few-shot learning, \citet{zhao2021calibrate} show the prediction of an LLM can depend on many seemingly irrelevant items, such as the prompt format or the order in which input examples are presented in a prompt, again with a sensititivity of performance to these factors. 
\citet{zhang2023deep} note that the topic structure of a document may be exchangeable, which motivates them to use Bayesian models, namely Latent Dirichlet Allocation, to analyse the representations of an LLM.
Our discussion on exchangeability relates to this line of work, but has a novel perspective on it through our focus on the martingale property, a necessary condition for exchangeability, among other implications of the martingale property which we study (e.g. the decomposition of uncertainty and the resulting identification of epistemic uncertainty).
Furthermore, in contrast to the related work, which shuffles the input data $Z_{1:n}$, we analyse the effect of shuffling the imputed, generated sequence $Z_{n+1}, \ldots$, where we find non-exchangeable behaviour which deviates from any reasonable Bayesian model.

\paragraph{Miscellaneous. }
Our work also relates a number of applications of LLMs.
As we are generating samples from an LLM with ICL, which as we demonstrate deviate from the distribution of the ICL dataset, this work relates to and has implications for a line of work on LLMs for synthetic data generation \cite{borisov2022language,hamalainen2023evaluating,tang2023does,veselovsky2023generating,li2023synthetic}.
Furthermore, we show that the martingale property is violated for long sampling paths, which may have implications for time series prediction with LLMs \cite{gruver2023large,jin2023time}, particularly over long horizons.
We also demonstrate a dependence on the order in which missing values are imputed, which has direct implications for the machine learning task of missing value imputations with LLMs \cite{mei2021capturing}.
\citet{shumailov2023model} demonstrate that models (including LLMs) which are recursively trained on data which they have previously generated shift in their distribution, where long tails disappear. 
While this work `conditions' on synthetic data by retraining, our work analyses the conditioning via ICL.
Lastly, as LLMs violate the martingale property in certain empirical regimes, they hence do not allow for a decomposed interpretation of their predictive uncertainty, which has important implications for uncertainty quantification with LLMs \cite{xiao2022uncertainty}.

\section{Negative Societal Impact}
\label{app:Negative societal impact}

This paper analyses and characterises the behaviour of LLMs. 
We try to understand whether ICL in LLMs follows Bayesian principles. 
As we outlined in \S \ref{sec:The martingale property enables a principled notion of uncertainty} this has important consequences for their potential use as trustworthy systems, which can be deployed in safety-critical, high-stakes applications such as healthcare. 
These systems often crucially rely on a principled notion of uncertainty. 
The evidence presented in this work cautions against the use of LLMs in such settings without further checks as they---under certain experimental settings---do not possess such a principled interpretation of uncertainty, rendering their uncertainty `black-box'.
Furthermore, while LLMs have typically been trained in non-exchangeable scenarios (e.g. natural language where the order of words or tokens changes meaning), 
as we showed in \S \ref{sec:martingale}, we caution against their use in exchangeable settings (e.g. i.i.d. in-context data) as their predictions can be rendered inconsistent.

The points noted above are potential negative societal impacts if Bayesian behaviour cannot be guaranteed by a model, as we argue in this work.
While we do not see any direct negative consequences from our analysis, we believe this work provides ample pointers and reason for further investigation of these concerns, and shall point out and warn against (potentially intended) misuse of LLMs.

\section{Code, Computational Resources, Datasets, Existing Assets Used}
\label{app:Code, computational resources, datasets, existing assets used}

\textbf{Code. }
We provide our code base on \href{https://github.com/meta-inf/bayes_icl}{\url{https://github.com/meta-inf/bayes_icl}} under MIT License, together with a \texttt{README.md} containing instructions on reproducing the key results in this paper.

\textbf{Datasets. }
We used three synthetic datasets for our experiments: a coin flip experiment, sampling from univariate Bernoulli distributions, a Gaussian experiment, sampling from univariate Gaussian distributions, and a synthetic natural language experiment, sampling (conditionally) from Bernoulli distributions.
We refer to \S \ref{sec:Experiments} and App. \ref{app:Additional experimental details and results} where they are introduced and discussed.

\textbf{Computational resources and APIs used. }
Referring to \S \ref{sec:Experiments}, we implemented \texttt{llama-2-7B} and \texttt{mistral-7B} with the Huggingface Transformer library \cite{wolf-etal-2020-transformers}, and  implemented \texttt{gpt-3}, \texttt{gpt-3.5} and \texttt{gpt-4} using the OpenAI API \cite{openai2023gpt4}.
For all Huggingface models, we generated the sampling paths by performing inference on a single A100 Nvidia GPU for each run.

\textbf{Existing assets used. }
Our work uses the following main software libraries and corresponding licenses:
PyTorch \cite{pytorch} (custom license), 
\texttt{numpy} \cite{harris2020array} (BSD 3-Clause License),
Weights\&Biases~\cite{wandb} (MIT License), 
Huggingface transformers library~\cite{wolf-etal-2020-transformers} (Apache License 2.0; model licenses see below),
\texttt{matplotlib} \cite{matplotlib} (PSF License),
\texttt{tqdm}   \cite{tqdm} (MPLv2.0 MIT License),
\texttt{scikit-learn} and \texttt{sklearn} \cite{scikit-learn} (BSD 3-Clause License),
\texttt{pandas} \cite{mckinney-proc-scipy-2010} (BSD 3-Clause License),
\texttt{openai} (Apache 2.0 License),  %
\texttt{tiktoken} (MIT License),   %
and \texttt{pickle} \cite{pickle} (License N/A).
We use Github Copilot  and ChatGPT \cite{openai2023gpt4} for code development and occasionally as a writing aid.

The five pretrained large language models we used (see \S \ref{sec:Experiments}) have the following licenses: 
\texttt{llama-2-7B} \cite{touvron2023llama} (custom license);
\texttt{mistral-7B} \cite{jiang2023mistral} (Apache 2.0 License); 
\texttt{gpt-3} \cite{brown2020language}, \texttt{gpt-3.5}, and \texttt{gpt-4} \citep{openai2023gpt4} (API; no code license).

\end{document}